\documentclass{article} 
\usepackage{iclr2026_conference,times}
\iclrfinalcopy

\usepackage{amsmath,amsthm,amsfonts,amssymb,bm}
\usepackage{algorithm,algpseudocode}
\usepackage{bbm,dsfont}
\usepackage{booktabs}
\usepackage{enumitem}
\usepackage{float}
\usepackage{hyperref,cleveref}
\usepackage{makecell}
\usepackage{pgfplots,pgfplotstable}
\usepackage{tikz}
\usepackage{pgfplots}
\usepackage{xcolor}
\usepackage{listings}

\usetikzlibrary{automata,arrows.meta,positioning}
\pgfplotsset{compat=1.17}
\usepgfplotslibrary{fillbetween}

\newtheorem{definition}{Definition}
\newtheorem{assumption}{Assumption}
\newtheorem{theorem}{Theorem}
\newtheorem{lemma}{Lemma}

\title{Frozen Policy Iteration: Computationally Efficient RL under Linear $Q^{\pi}$ Realizability for Deterministic Dynamics}

\newcommand{\affiliate}[1]{\textsuperscript{\textnormal{#1}}}
\author{Yijing Ke\affiliate{1},\quad Zihan Zhang\affiliate{2},\quad Ruosong Wang\affiliate{3} \\[1ex]
\affiliate{1}School of EECS, Peking University \\
\affiliate{2}Department of Computer Science and Engineering, HKUST \\
\affiliate{3}CFCS and School of
Computer Science, Peking University \\
\texttt{keyijing@stu.pku.edu.cn} \\
\texttt{zihanz@cse.ust.hk} \\
\texttt{ruosongwang@pku.edu.cn}
}

\allowdisplaybreaks[4]

\def\A{\mathcal{A}}
\def\D{\mathcal{D}}
\def\E{\mathbb{E}}
\def\F{\mathcal{F}}
\def\N{\mathbb{N}}
\def\S{\mathcal{S}}
\def\P{\mathbb{P}}
\def\R{\mathbb{R}}

\newcommand{\argmin}{\mathop{\arg\min}}
\newcommand{\argmax}{\mathop{\arg\max}}
\newcommand{\cover}{\operatorname{Cover}}
\newcommand{\poly}{\operatorname{poly}}
\newcommand{\numberthis}{\addtocounter{equation}{1}\tag{\theequation}}

\counterwithin*{ALG@line}{algorithm}

\begin{document}

\maketitle

\begin{abstract}
We study computationally and statistically efficient reinforcement learning under the linear $Q^{\pi}$ realizability assumption, where any policy's $Q$-function is linear in a given state-action feature representation. Prior methods in this setting are either computationally intractable, or require (local) access to a simulator. In this paper, we propose a computationally efficient online RL algorithm, named \emph{Frozen Policy Iteration}, under the linear $Q^{\pi}$ realizability setting that works for Markov Decision Processes (MDPs) with stochastic initial states, stochastic rewards and deterministic transitions. Our algorithm achieves a regret bound of $\widetilde{O}(\sqrt{d^2H^6T})$, where $d$ is the dimensionality of the feature space, $H$ is the horizon length, and $T$ is the total number of episodes. Our regret bound is optimal for linear (contextual) bandits which is a special case of our setting with $H = 1$.

Existing policy iteration algorithms under the same setting heavily rely on repeatedly sampling the same state by access to the simulator, which is not implementable in the online setting with stochastic initial states studied in this paper.  In contrast, our new algorithm circumvents this limitation by strategically using only high-confidence part of the trajectory data and freezing the policy for well-explored states, which ensures that all data used by our algorithm remains effectively \emph{on-policy} during the whole course of learning. 
We further demonstrate the versatility of our approach by extending it to the Uniform-PAC setting and to  function classes with bounded eluder dimension.  

\end{abstract}
\section{Introduction}
In modern reinforcement learning (RL), function approximation schemes are often employed to handle large state spaces. During the past decade, significant progress has been made towards understanding the theoretical foundation of RL with function approximation. By now, a rich set of tools and techniques have been developed~\citep{jiang2017contextual, sun2019model, jin2020provably, wang2020reinforcement, du2021bilinear, jin2021bellman, foster2021statistical}, which have led to statistically efficient RL with function approximation under various structural assumptions. 
However, despite these advancements, a remaining challenge is that many of these statistically efficient RL algorithms are computationally inefficient, which severely limits the practical applicability of these approaches. 
Recently, these computational-statistical gaps have drawn significant interest from the RL theory community, leading to computationally efficient algorithms~\citep{zanette2020provably, wu2024computationally, golowich2024linear} and computational hardness results~\citep{kane2022computational}. 

For the special case of RL with linear function approximation, a structural assumption common in the literature is the linear bellman completeness assumption~\citep{zanette2020learning, zanette2020provably, wu2024computationally, golowich2024linear}, which, roughly speaking, assumes that the Bellman backups of linear state-action value functions are still linear w.r.t.~a fixed feature representation given to the learner.
This assumption naturally arises when analyzing value-iteration type RL algorithms (e.g., FQI~\citep{munos2005error}) with linear function approximation. 
It is by now well-understood that RL with linear bellman completeness is statistically tractable~\citep{zanette2020learning, du2021bilinear, jin2021bellman}, in the sense that a near-optimal policy can be learned with polynomial sample complexity. 
However, the algorithm employed to achieve such polynomial sample complexity require solving computationally intractable optimization problems, and recently a number of algorithms have been proposed to fill this computational-statistical gap, leading to statistically and computationally efficient RL under  linear bellman completeness for MDPs satisfying the explorability assumption~\citep{zanette2020provably}, MDPs with deterministic dynamics, stochastic rewards and stochastic initial states~\citep{wu2024computationally}, and MDPs with constant number of actions~\citep{golowich2024linear}. 

A major drawback of the linear bellman completeness assumption, as observed by~\citet{chen2019information}, is that it is \emph{not monotone} in the function class under consideration. For the linear function case, adding more features into the feature representation may break the assumption. 
This makes the bellman completeness much less desirable for practical scenarios, as in modern RL applications, neural networks with billions of parameters are employed as the function approximators, and feature selection could be intractable in such cases.

For RL with linear function approximation, another common assumption is the linear $Q^{\pi}$ realizability~\citep{du2019good, lattimore2020learning, yin2022efficient, weisz2022confident, weisz2023online}, which assumes that any policy's $Q$-function is linear in a fixed state-action feature representation given to the learner. 
This assumption naturally arises when analyzing policy-iteration type RL algorithms (e.g., LSPI~\citep{lagoudakis2003least}) with linear function approximation, and it has the desired monotonicity property, in the sense that adding more features never hurts realizability. 
However, unlike linear bellman completeness, the computational-statistical gap under linear $Q^{\pi}$ realizability is largely unexplored.
Under this assumption, the first method with polynomial sample complexity is due to~\citet{weisz2023online}, and similar to linear bellman completeness, existing algorithms~\citep{weisz2023online,mhammedi2025sample} with such sample complexity requires either computationally intractable optimization problems or oracles. 
Other algorithms~\citep{du2019good, lattimore2020learning, yin2022efficient, weisz2022confident} with polynomial running time further require (local) access to a simulator, which enables the algorithm to restart from any visited states.
In contrast, in the standard online RL setting, there is no known RL algorithm that is both statistically and computationally efficient.
In fact, under the linear $Q^{\pi}$ realizability, even when the transition dynamics are deterministic, it is unclear if  computationally efficient RL is possible.

In terms of methodology, the algorithm by~\citet{weisz2023online} employs the {\em global optimism} approach and relies on maintaining large and complex version spaces. It is unclear how to implement such approach in a computationally efficient manner, and obtaining  an algorithm with polynomial running time under linear $Q^{\pi}$ realizability was left as an open problem by~\citet{weisz2023online}. 
\citet{mhammedi2025sample} achieves sample and oracle efficient under linear $Q^{\pi}$ realizability by relying on a cost-sensitive classification oracle. However, such oracle could be NP-hard to implement in the worst case. 
Meanwhile, the algorithms by~\cite{yin2022efficient, weisz2022confident} employ the standard approximate policy-iteration framework and assume local access to a simulator. In these algorithms, before adding a state-action pair $(s, a)$ into the dataset, the algorithms use multiple rollouts starting from $(s, a)$ to ensure all successor states are well-explored, so that the policy value associated with $(s, a)$ is accurate. 
Once a new under-explored state appears, the whole algorithm restarts from that state and perform additional rollouts to make it well-explored. 
Consequently, these algorithms critically rely on the simulator to revisit the same state for multiple times. 
However, such resampling mechanism is not implementable in the standard online RL setting when the state-space is large.
Even with deterministic dynamics, as long as the initial state is stochastic, we might not even encounter the same state twice during the whole learning process. 

In this paper, we provide the first computationally efficient algorithm under the linear $Q^{\pi}$ realizability assumption with stochastic initial states, stochastic rewards and deterministic dynamics in the online RL setting. 
Our algorithm, named \emph{Frozen Policy Iteration} (FPI), achieves a regret bound of $\widetilde{\mathcal{O}}(\sqrt{d^2H^6T})$\footnote{Throughout the paper, we use $\widetilde{\mathcal{O}}(\cdot)$to suppress logarithm factors. }, where $d$ is the dimensionality of the feature space, $H$ is the horizon length, and $T$ is the total number of episodes, which is optimal for linear (contextual) bandits~\citep{dani2008stochastic}, a special case with $H = 1$. 
Unlike existing policy iteration algorithms~\citep{du2019good, lattimore2020learning, yin2022efficient, weisz2022confident}, our new algorithm circumvents the resampling issue by strategically using only high-confidence part of the trajectory data and {\em freezing} the policy for well-explored states, which ensures that all data used by our algorithm remains effectively {\em on-policy} during the whole course of learning. 
We further demonstrate the versatility of our approach by extending it to the Uniform-PAC setting~\citep{dann2017unifying} and to  function classes with bounded eluder dimension~\citep{russo2013eluder}.  
Due to the simplicity of our algorithm, we are able to give a proof-of-concept implementation on standard control tasks to illustrate practicality and to ablate the role of freezing.

\section{Related Work}

\paragraph{Linear $Q^{\pi}$ Realizability. } The connection and difference between our work and prior work under Linear $Q^{\pi}$ Realizability~\citep{du2019good, lattimore2020learning, yin2022efficient, weisz2022confident, weisz2023online} have already been discussed in the introduction. 
In Table~\ref{table:compare}, we compare our new result with prior works to make the difference clear. 
We also note that the algorithms by~\citet{lattimore2020learning, yin2022efficient, weisz2022confident} work in the discounted setting, while our new algorithm and the algorithm by~\citet{du2019good, weisz2023online} work under the finite-horizon setting.  

\paragraph{Linear $Q^*$ Realizability. } The Linear $Q^*$ Realizability assumes that the optimal $Q$-function is linear in a given state-action feature representation. 
Since Linear $Q^{\pi}$ Realizability implies Linear $Q^*$ Realizability, algorithms under Linear $Q^*$ Realizability could also be relevant to our setting. 
However, existing algorithms that work under this assumption either require lower bounded suboptimality gap and deterministic dynamics with deterministic initial states~\citep{du2019provably, du2020agnostic}, fully deterministic system (i.e., deterministic rewards and initial states)~\citep{wen2013efficient}, or lower bounded suboptimality gap with local access to a simulator~\citep{li2021sample}. 
All these assumptions are much stronger than those assumed by prior algorithms under Linear $Q^{\pi}$ Realizability.
\paragraph{Uniform-PAC. } \citet{dann2017unifying} propose the Uniform-PAC setting as a bridge between the  PAC setting and the regret minimization setting, and develop the first tabular RL algorithm with Uniform-PAC guarantees. 
Uniform-PAC guarantees are also achieved for RL with function approximation~\citep{he2021uniform, wu2023uniform}. 
To achieve the Uniform-PAC guarantee, we use an accuracy level framework similar to that of~\citet{he2021uniform}, though the details are substantially different since our algorithm is based on policy-iteration, while their algorithm is based on value-iteration. 

\begin{table}[htbp]
    \centering
    \begin{tabular}{cccccc}
    \toprule
     & Access model & \makecell{Stochastic \\ transitions} & \makecell{PAC \\ guarantee} & \makecell{Regret \\ guarantee} \\
    \hline
    \citet{du2019good} & generative model & $\checkmark$ & $\checkmark$ & $\times$ \\
    \hline
    \citet{lattimore2020learning} & generative model & $\checkmark$ & $\checkmark$ & $\times$ \\
    \hline
    \citet{yin2022efficient} & local access & $\checkmark$ & $\checkmark$ & $\times$ \\
    \hline
    \citet{weisz2022confident} & local access & $\checkmark$ & $\checkmark$ & $\times$ \\
    \hline
    \citet{weisz2023online} & \makecell{online RL \\ (computationally \\ intractable)} & $\checkmark$ & $\checkmark$ & $\times$ \\
    \hline
    This work & online RL & $\times$ & $\checkmark$ & $\checkmark$ \\
    \bottomrule
    \end{tabular}
    \caption{Comparison with prior works.}
    \label{table:compare}
\end{table}

\section{Preliminaries}

A finite-horizon Markov Decision Process (MDP) is defined by the tuple $(\S,\A,H,P,R,\mu)$. $H\in\N^+$ is the horizon. 
$\S=\S_1\cup\cdots\cup\S_H$ is the state space, partitioned by the horizon $H$, with $\S_h$ denoting the set of states at stage $h\in[H]:=\{1,\cdots,H\}$. 
We assume that $\S_1,\cdots,\S_H$ are disjoint sets. 
$\A$ is the action space. 
$P:\S_h\times\A\to\Delta(\S_{h+1})$ is the transition function, where $\Delta(\S_{h+1})$ denotes the set of probability distributions over $S_{h+1}$. 
$R:\S\times\A\to\Delta([0,1])$ is the reward distribution, where $R(s,a)$ is a distribution on $[0,1]$ with mean $r(s,a)$, indicating the stochastic reward received by taking action $a$ on state $s$. 
$\mu\in\Delta(S_1)$ is the probability distribution of the initial state. 
Given a policy $\pi:\S\to\A$, for each $s\in\S_h$ and $a\in\A$, we define the state-action value function of policy $\pi$ as $Q^{\pi}(s,a)=\E\left[\left.\sum_{i=h}^Hr_i\right|s_h=s,a_h=a,\pi\right]$ 
and the state value function as $V^{\pi}(s)=Q^{\pi}(s,\pi(s))$. Let $\pi^*$ be the optimal policy.

For an RL algorithm, define its regret in the first $T$ episodes as the sum of the suboptimality gap of the first $T$ trajectories,
$\operatorname{Reg}(T)=\sum_{t=1}^T\left(V^{\pi^*}(s_1^{(t)})-\sum_{h=1}^H r(s_h^{(t)},a_h^{(t)})\right)$, 
where $s_1^{(t)},a_1^{(t)},\cdots,s_H^{(t)},a_H^{(t)}$ is the trajectory in the $t$-th episode.

Let $V\in\R^{d\times d}$ be symmetric positive definite. For any $x\in\R^d$ define the elliptical norm induced by $V$ as $\|x\|_V = \sqrt{x^\top V x}$. 
A zero-mean random variable $X$ is called $\sigma$-subGaussian if $\E\left[e^{\lambda X}\right]\le e^{\lambda^2\sigma^2/2}$ for all $\lambda\in\R$.

In this work, we assume the Linear $Q^{\pi}$ Realizability assumption, formally stated below. 
\begin{assumption}[Linear $Q^{\pi}$ Realizability]\label{assump:realizability}
    MDP $(\S,\A,H,P,R,\mu)$ is $\kappa$-approximate $Q^\pi$ realizable with a feature map $\phi:\S\times\A\to\R^d$ for some $\kappa>0$. That is, for any policy $\pi$, there exists $\theta_1^\pi,\cdots,\theta_H^\pi\in\R^d$ such that for any $h\in[H],s\in\S_h$ and $a\in\A$,
    $
    \left|Q^\pi(s,a)-\left\langle\phi(s,a),\theta_h^\pi\right\rangle\right|\le\kappa.
    $
\end{assumption}
We assume that the feature map $\phi$ is given to learner. Since $Q^\pi(s,a)\in[0,H]$, we also make the following assumption, which ensures that $\phi(s, a)$ and $\theta_h^\pi$ have bounded norm for all $(s, a) \in \S \times \A$ and $h \in [H]$. 
\begin{assumption}[Boundedness]\label{assump:bound}
    For any $h\in[H]$, $\Vert\phi(s,a)\Vert_2\le 1$ for all $s\in\S_h,a\in\A$, and $\left\Vert\theta_h^\pi\right\Vert_2\le\sqrt{d}H$ for any policy $\pi$.
\end{assumption}
We further require deterministic state transition.
\begin{assumption}[Deterministic Transitions]
    For any $s\in\S$ and $a\in\A$, $P(s,a)$ is a one-point distribution, i.e., the system transitions to a unique state after taking action $a$ on state $s$.
\end{assumption}

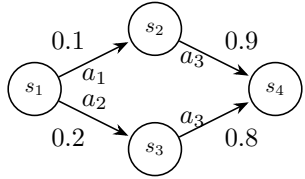
\begin{figure}[htbp]
    \centering
    \begin{tikzpicture}[->,>=Stealth,shorten >=1pt,auto,node distance=2cm,semithick]
        \node[state,scale=0.8] (1) {$s_1$};
        \node[state,scale=0.8,right of=1,yshift=1cm] (2) {$s_2$};
        \node[state,scale=0.8,right of=1,yshift=-1cm] (3) {$s_3$};
        \node[state,scale=0.8,right of=2,yshift=-1cm] (4) {$s_4$};
        \draw (1) edge node {$0.1$} node[below] {$a_1$} (2);
        \draw (1) edge node[below left] {$0.2$} node[above] {$a_2$} (3);
        \draw (2) edge node {$0.9$} node[left] {$a_3$} (4);
        \draw (3) edge node[below right] {$0.8$} node[left] {$a_3$} (4);
    \end{tikzpicture}
    \caption{An example of MDP satisfying linear $Q^{\pi}$ realizability, where $\phi(s_1,a_1)=\phi(s_1,a_2)=e_1$, $\phi(s_2,a_3)=e_2$, $\phi(s_3,a_3)=e_3$.}
    \label{example:lin_q_pi}
\end{figure}

We note that linear $Q^\pi$ realizability does not imply linear bellman completeness even in MDPs with deterministic transitions.
Figure~\ref{example:lin_q_pi} shows an example satisfying linear $Q^\pi$ realizability which is not linear bellman complete.
Moreover, although we assume the transition to be deterministic, the initial state distribution is stochastic (in fact, our algorithm allows for adversarially chosen initial states), and the reward signal could be stochastic. Therefore, learning is still challenging in this case.
Also, stochastic initial state could in fact cover a number of standard RL benchmarks including classical control tasks (e.g., CartPole, Acrobot)~\citep{towers2024gymnasium} and MuJoCo Environments~\citep{todorov2012mujoco}, where randomness primarily arises from the reset distribution rather than process noise. For Atari Games, when sticky actions are disabled, the simulator yields deterministic dynamics; stochasticity is typically injected only at reset. 

\newcommand{\frozen}{\texttt{Frozen}}

\section{Warmup: Frozen Policy Iteration in the PAC Setting}\label{sec:pac}
In this section, we present the probably approximately correct (PAC) version of our main algorithm, which serves as a warmup of the more complicated regret minimization algorithm in Section~\ref{sec:regret}.

\subsection{The Algorithm}

\begin{algorithm}[htbp]
\caption{Frozen Policy Iteration-PAC (FPI-PAC)}
\label{alg:pac}
\begin{algorithmic}[1]
\State Initialize $\D_{1, h} = \{\}$ for all $h \in [H]$
\For{$t$ = $1,\cdots,T$}
   \State For each $h \in [H]$ and $s \in \S_h $, define 
	\[
	\pi_t(s)=\begin{cases}
    \argmax_{a\in\A}Q_t(s,a) & \left(s,a\right)\in\cover\left(\D_{t, h},\varepsilon\right)\text{ for all }a\in\A\\
    \text{any $a\in\A$ such that $(s,a)\notin\cover\left(\D_{t, h},\varepsilon\right)$} & \text{otherwise}
\end{cases},
    \]
    where $\cover\left(\D_{t, h},\varepsilon\right)$ is as defined in \eqref{eq:cover} and $Q_t(s,a)$ is as defined in \eqref{eq:Qt}
    \State Receiving $s_1^{(t)}, a_1^{(t)}, r_1^{(t)} ,s_{2}^{(t)},  a_2^{(t)}, r_2^{(t)}, \ldots, s_H^{(t)}, a_H^{(t)}, r_H^{(t)}$ by executing $\pi_t$
    \State $h_t\gets\max\left\{h\in[H]:(s_h^{(t)},a_h^{(t)})\notin\cover(\D_{t, h},\varepsilon)\right\}$ \label{step:h_t}
    \If{$h_t$ exists}
        \State $\hat{q}_t\gets\sum_{h=h_t}^H r_h^{(t)}$
        \State Update $\D_{t + 1, h_t} \gets \D_{t, h_t}  \cup \left \{ \left(s_{h_t}^{(t)},a_{h_t}^{(t)},\hat{q}_t\right) \right\}$ and $\D_{t + 1, h} \gets \D_{t, h}$ for any $h \neq h_t$
    \EndIf
\EndFor
\end{algorithmic}
\end{algorithm}

\paragraph{Datasets.} In our algorithm, for each step $h \in [H]$, we maintain a dataset $\D_h$ which is an ordered sequence consisting of $\left(s_{h,i},a_{h,i},q_{h,i}\right)$, where $(s_{h,i},a_{h,i}) \in \S_h \times \A$ is a state-action pair, and $q_{h,i}$ is total reward obtained by following a policy $\pi_t$ starting from $(s_{h,i},a_{h,i})$.
We further define $\phi_{h,i}=\phi\left(s_{h,i},a_{h,i}\right)$ to be the feature of $(s_{h,i},a_{h,i})$. 
In the description of Algorithm~\ref{alg:pac}, for clarity, we use $\D_{t, h}$ to denote the snapshot of $\D_h$ {\em before} the $t$-th round, and we initialize $\D_{1, h}$, i.e., the dataset before the first round, to be the empty set for all $h \in [H]$. 

\paragraph{Exploration Mechanism.} For each round $t$, we first define a policy $\pi_t$, where for each state $s$ at step $h \in [H]$, we first test if for all actions $a \in \mathcal{A}$, $(s, a)$ could be covered by existing data in $\D_{t, h}$. Concretely, for a dataset $\D=\left\{(s_i,a_i, q_i)\right\}_{i=1}^n$, define
\begin{equation}\label{eq:cover}
\cover(\D,\varepsilon)=\left\{(s,a)\in\S\times\A:\Vert\phi(s,a)\Vert_{\Sigma^{-1}}\le\varepsilon\right\}
\end{equation}
where $\Sigma=\lambda I+\sum_{i=1}^n\phi(s_i,a_i)\phi(s_i,a_i)^\top$ is the regularized empirical feature covariance matrix.
Equivalently, $\cover(\D,\varepsilon)$ contains all state-action pairs $(s, a)$ for which the least squares estimate has error upper bounded by (roughly) $\varepsilon$. 
If for all actions $a \in \mathcal{A}$, $(s, a)$ lies in $\cover(\D,\varepsilon)$, we define $\pi_t(s)$ to be the greedy policy with respect to a $Q$-function $Q_t$ (defined later). 
Otherwise, there must be an action $a \in \A$ satisfying $(s, a) \notin \cover(\D,\varepsilon)$, and we define $\pi_t(s)$ to be any such $a$ to encourage exploration. 
We break tie in a consistent manner when there are multiple such $a$.

\paragraph{Avoiding Off-Policy Data and Resampling. }After defining a policy $\pi_t$, we execute $\pi_t$ to receive a trajectory sample,  forms a new dataset $\D_{t + 1, h}$ based on $\D_{t, h}$ and the collected data, and proceeds to the next round $t + 1$.
So far, the design of our algorithm does not differ significantly from the standard policy iteration framework. 
Indeed, the main novelty of Algorithm~\ref{alg:pac} lies in how we form the new dataset $\D_{t + 1, h}$ and define a $Q$-function $Q_{t + 1}$ based on $\D_{t + 1, h}$.
If not handled properly, our datasets could contain {\em off-policy} data once $\pi_t$ is updated.
In particular, for each data $\left(s_{h,i},a_{h,i},q_{h,i}\right)$, $q_{h, i}$ might deviate significantly from the $Q$-value of $(s_{h,i},a_{h,i})$ once the policy is updated. 
We note that assuming access to a simulator~\citep{du2019good, lattimore2020learning, yin2022efficient, weisz2022confident, weisz2023online}, one could simply recollect samples by executing the updated policy
starting from $(s_{h,i},a_{h,i})$ to estimate its $Q$-value with respect to the updated policy. 
Unfortunately,  in the online setting with stochastic initial states, such resampling is not implementable as we might not even encounter the same state twice during the whole learning process. 

\paragraph{Updating Datasets.} In the design of Algorithm~\ref{alg:pac}, we only include the high-confidence part of the trajectory data into the new datasets.
In particular, in Step~\ref{step:h_t} of Algorithm~\ref{alg:pac}, we define $h_t$ to be last step $h$, so that $(s_h^{(t)}, a_h^{(t)})$ could be not covered by existing data. We then add $(s_{h_t}^{(t)}, a_{h_t}^{(t)}, \hat{q}_t)$ into $\D_{t, h_t}$ to form $\D_{t + 1, h_t}$ where $\hat{q}_t$ is the total reward starting from step $h_t$. 
For all other step $h \neq h_t$, we keep $\D_{t, h}$ unchanged. 
This design choice is due to the following reason: for all steps $h > h_t$, we must have $(s_h^{(t)}, a) \in \cover\left(\D_{t, h},\varepsilon\right)$ for all $a \in \A$, since otherwise an exploratory action would have been chosen (cf.~the definition of $\pi_t$), which further implies that $\pi_t$ is a near-optimal policy for $s_{h_t + 1}^{(t)}, s_{h_t + 2}^{(t)}, \ldots, s_{H}^{(t)}$. 
On the other hand, since $(s_{h_t}^{(t)}, a_{h_t}^{(t)}) \notin \cover\left(\D_{t, h_t},\varepsilon\right)$, $a_{h_t}^{(t)}$ could be a suboptimal action for $s_{h_t}^{(t)}$. 
Because of this, we only include $(s_{h_t}^{(t)}, a_{h_t}^{(t)})$ into the new dataset, and discard all other state-action pairs on the trajectory. 

\paragraph{Freezing High-Confidence States.} It remains to define the $Q$-function $Q_t$ based on the dataset $\D_{t, h}$. 
Here, a na\"ive choice is to use all data in $\D_{t, h}$ to obtain a least-squares estimate and define $Q_t$ accordingly, but by doing so, eventually our datasets could still contain {\em off-policy} data once the policy $\pi_t$ is updated. 
In particular, although $s_h^{(t)}$ lies in the high-confidence region coverd by $\D_{t, h}$ for all $h > h_t$, the policy on state $s_h^{(t)}$ could still be changed in later rounds once we update $\pi_t$, which means our dataset contains reward sums from different policies. 
Here, our new idea is to {\em freeze} the updates on $\pi_t(s)$, once our datasets could cover $(s, a)$ for all $a \in \A$. 
Concretely, for each round $t$, for each $s \in \S_h$, we define 
\begin{equation}
    k_t(s)=\min\left\{k\in\N:(s,a)\in\cover\left(\left\{(s_{h,i},a_{h,i},q_{h, i})\right\}_{i=1}^k,\varepsilon\right)\text{ for all }a\in\A\right\}\wedge |\D_{t,h}|.\label{eq:frozen}
\end{equation}
Recall that $\D_{t, h} = \{(s_{h,i},a_{h,i},q_{h, i})\}_{i = 1}^{|\D_{t, h}|}$ is the snapshot of our dataset at step $h$ before the $t$-th round. Essentially, $k_t(s)$ is the first time we add a data into $\D_{t, h}$, after which $(s, a)$ could be covered by $\D_{t, h}$ for all $a \in \A$. 
Therefore, by defining $Q_t(s, \cdot )$ using only the first $k_t(s)$ data in $\D_{t, h}$, we effectively freeze the updates on $\pi_t(s)$ once our datasets could cover $(s, a)$ for all $a \in \A$. Therefore, for each $h \in [H]$ and $s\in\S_h$, we define 
\begin{equation}\label{eq:Qt}
    Q_t(s,a)=\tilde{Q}_{k_t(s)}(s,a),
\end{equation}
where $k_t(s)$ is as defined in \eqref{eq:frozen} and 
\[
    \tilde{Q}_k(s,a)=\left\langle\phi(s,a),\Sigma_{h,k}^{-1}\sum_{i=1}^k\phi_{h,i}q_{h,i}\right\rangle
\]
is the least squares estimate by using only the first $k$ data in $\D_h$ with 
\begin{equation}\label{eq:cov_h_k}
\Sigma_{h,k}=\lambda I+\sum_{i=1}^k\phi_{h,i}\phi_{h,i}^\top.
\end{equation}

\subsection{The Analysis}
In this section, we outline the analysis of Algorithm~\ref{alg:pac}. 
In our analysis, define 
\begin{equation}\label{eq:D_bound}
D=\frac{2d}{\varepsilon^2}\ln\left(1+\frac{4\varepsilon^{-4}}{\lambda^2}\right) .
\end{equation}
Our first lemma shows that 
for any round $t$ and $h \in [H]$, $|\D_{t,h}|$ is always upper bounded by $D$.
\begin{lemma}\label{lem:pac_dataset_size_ub}
    For any $t\ge 1,h\in[H]$, it holds that $|\D_{t,h}|\le D$.
\end{lemma}
The proof of Lemma~\ref{lem:pac_dataset_size_ub} is based on the observation that if we add a data $(s_{h, i}, a_{h, i}, q_{h, i})$ into $\D_{t + 1, h}$ in the $t$-th round, we must have $(s_{h, i}, a_{h, i}) \notin\cover(\D_{t, h},\varepsilon)$. Therefore, by the standard elliptical potential lemma~\citep{lattimore2020bandit}, Lemma~\ref{lem:pac_dataset_size_ub} holds.

Our second lemma shows that for a state-action pair $(s_{h,i},a_{h,i})$ added into $\D_h$ in the $t$-th round, the $Q$-function of $(s_{h,i},a_{h,i})$ with respect to any later policy $\pi_{t'}$ (i.e., $t' \ge t$) will be unchanged. 
\begin{lemma}\label{lem:on_policy}
For any $h\in[H]$ and $i\ge 1$, let $t$ be the round when $(s_{h,i},a_{h,i},q_{h,i})$ is appended to $\D_h$, i.e., $(s_{h,i},a_{h,i},q_{h,i}) \in \D_{t + 1, h}$ while $(s_{h,i},a_{h,i},q_{h,i}) \notin \D_{t, h}$.
Then $Q^{\pi_t}(s_{h,i},a_{h,i})=Q^{\pi_{t^\prime}}(s_{h,i},a_{h,i})$ for all $t^\prime\ge t$.
\end{lemma}
The proof of Lemma~\ref{lem:on_policy} is based on the following observation: suppose we add a state-action pair $(s_{h_t}^{(t)}, a_{h_t}^{(t)})$ into the dataset in the $t$-th round, for all $h > h_t$, we must have $(s_h^{(t)}, a) \in \cover\left(\D_{t, h},\varepsilon\right)$ for all $a \in \A$, since otherwise an exploratory action $a$ with $(s_h^{(t)}, a) \notin \cover\left(\D_{t, h},\varepsilon\right)$ would have been chosen. 
Therefore, by the way we define $Q_t$ (cf.~\eqref{eq:Qt}), the action of $s_h^{(t)}$ will be frozen for policies $\pi_{t'}$ defined in later rounds $t' > t$. In short, Lemma~\ref{lem:on_policy} formalizes the intuition that all data used by our algorithm remain effectively {\em on-policy} even if the policy $\pi_t$ is updated.  

 For each round $t\ge 1$, let $\xi_t = \hat{q}_t-Q^{\pi_t}(s_{h_t}^{(t)},a_{h_t}^{(t)})$. 
 Because both $\pi_t$ and the MDP transitions are deterministic, $s_h^{(t)}$ and $a_h^{(t)}$ are determined given $\pi_t$. Consequently, $\xi_t$ is the sum of at most $H$ independent $1$-subGaussians, and is therefore $\sqrt{H}$-subGaussian. 
 Our third lemma defines a high probability event, which we will condition on in later parts of the proof. 
 Its proof is based on standard concentration inequalities for self-normalized processes~\citep{abbasi2011improved}. 
\begin{lemma}\label{lem:e}
	Define event $\mathfrak{E}_\text{high}$ as
	\[
	\left\{\text{$\forall h\in[H],k\ge 1$, $\left\Vert\sum_{i=1}^k\phi_{h,i}\xi_{t_{h,i}}\right\Vert_{{\Sigma_{h,k}}^{-1}}^2\le 2H\left(\frac{d}{2}\ln\left(1+\frac{k}{\lambda d}\right)+\ln\frac{H}{\delta}\right)$}\right\} .
	\]
	Then $\P[\mathfrak{E}_\text{high}]\ge 1-\delta$.
\end{lemma}

Define
\[
	\alpha=\sqrt{2H\left(\frac{d}{2}\ln\left(1+\frac{D}{\lambda d}\right)+\ln\frac{H}{\delta}\right)}+\sqrt{D}\kappa+\sqrt{\lambda d}H, 
\]
where $D$ is as defined in \eqref{eq:D_bound}. Lemma~\ref{lem:tilde_Q} show that under the event $\mathfrak{E}_\text{high}$ defined in Lemma~\ref{lem:e}, $|\tilde{Q}_k(s,a)-Q^{\pi_t}(s,a)|$ is upper bounded by $\alpha\Vert\phi(s,a)\Vert_{{\Sigma_{h,k}}^{-1}}+\kappa$.

\begin{lemma}\label{lem:tilde_Q}
	Under $\mathfrak{E}_\text{high}$, for any $t\ge 1,h\in[H],0\le k\le |\D_{t,h}|,s\in\S_h,a\in\A$, it holds that
	$
	|\tilde{Q}_k(s,a)-Q^{\pi_t}(s,a)|\le\alpha\Vert\phi(s,a)\Vert_{{\Sigma_{h,k}}^{-1}}+\kappa,
	$
	where $\Sigma_{h,k}$ is as defined in~\eqref{eq:cov_h_k}.
\end{lemma}
The proof of Lemma~\ref{lem:tilde_Q} uses the high probability event in Lemma~\ref{lem:e}, and also critically relies on Lemma~\ref{lem:on_policy} which shows that all data in our dataset remain on-policy during the whole algorithm. 

The following lemma, which is a direct implication of Lemma~\ref{lem:tilde_Q} and the definition of $\cover(\D_{t,h},\varepsilon)$ (cf.~\eqref{eq:cover}) and $Q_t$ (cf.~\eqref{eq:Qt}), shows that for those state-action pairs $(s,a)\in\cover(\D_{t,h},\varepsilon)$, we have $\left|Q_t(s,a)-Q^{\pi_t}(s,a)\right|\le\alpha\varepsilon+\kappa$. 
\begin{lemma}\label{lem:Qt}
	Under $\mathfrak{E}_\text{high}$, for any $t\ge 1,h\in[H],s\in\S_h,a\in\A$, if $(s,a)\in\cover(\D_{t,h},\varepsilon)$, then 
	$
	\left|Q_t(s,a)-Q^{\pi_t}(s,a)\right|\le\alpha\varepsilon+\kappa .
	$
\end{lemma}

Our final lemma characterizes the suboptimality of $\pi_t$ for those well-explored states $s\in\S$ satisfying $(s,a)\in\cover(\D_{t,h},\varepsilon)$ for all $a\in\A$. 
\begin{lemma}\label{lem:sub}
    Under $\mathfrak{E}_\text{high}$, given $t\ge 1,s\in\S_1$, if $(s,a)\in\cover(\D_{t,1},\varepsilon)$ for all $a\in\A$, then we have
    $
    V^{\pi_t}(s)\ge V^{\pi^*}(s)-2H(\alpha\varepsilon+\kappa) .
    $
\end{lemma}
To prove Lemma~\ref{lem:sub}, by the performance difference lemma~\citep{kakade2002approximately}, we only need to show that for the state-action pair $(s_{h}^*, a_{h}^*)$ at each step $h$ of the trajectory induced by the optimal policy $\pi^*$, switching from $a_{h}^* = \pi^*(s_{h}^*)$ to $\pi_t(s_{h}^*)$ reduces the policy value of $\pi_t$ by at most $2(\alpha\varepsilon+\kappa)$. If for all actions $a \in \A$, $(s_{h}^*, a)$ could be covered by $\D_{t, h}$, then the above claim is clearly true by Lemma~\ref{lem:Qt} and the definition of $\pi_t$ and $Q_t$ in~\eqref{eq:Qt}. 
Otherwise, we could find a well-explored state $s_{h'}^*$ for some $h' < h$.
Using the $Q^{\pi}$ realizability assumption (Assumption~\ref{assump:realizability}) for such $s_{h'}^*$, we further shows the policy value of any $a \in A$ at state $s_{h}^*$ differs by at most $\alpha\varepsilon+\kappa$. 

By taking $\lambda=H^{-1}$ and scaling $\varepsilon$ by a factor of $\Theta(H\alpha)$ (denoted by $\overline{\varepsilon}$), we arrive at the final guarantee of Algorithm~\ref{alg:pac}. 
\begin{theorem}\label{thm:pac_main}
    There is a constant $C$ so that for any $\overline{\varepsilon}>C\sqrt{d}H\kappa$, by picking certain $\varepsilon$, with probability at least $1-\delta$, the number of episodes with suboptimality gap greater than $\overline{\varepsilon}$ is at most $\widetilde{\mathcal{O}}\left(\frac{d^2H^4}{\overline{\varepsilon}^2}\right)$.
\end{theorem}

To prove Theorem~\ref{thm:pac_main}, note that if $(s_1^{(t)},a)\in\cover(D_{t,1},\varepsilon)$ for all $a\in\A$, Lemma~\ref{lem:sub} guarantees that the suboptimality gap of such episode is at most $\overline{\varepsilon}$. 
In the other episodes, we must have increased the size of one dataset by $1$, whose total occurrence count is upper bounded by $HD$ because every dataset will be updated at most $D$ times by Lemma~\ref{lem:pac_dataset_size_ub}. 

\paragraph{Time and Space Complexity.} Assume that $\phi(s,a)$ can be computed in $\poly(d)$ time and space given any $s\in\S$ and $a\in\A$. In Algorithm~\ref{alg:pac} the size of every dataset is bounded by $D$, so its space complexity is $\widetilde{\mathcal{O}}\left(\frac{H\poly(d)}{\varepsilon^2}\right)$. 
We assume that the action space is finite, under which case $\pi_t(s)$ can be computed in time $\mathcal{O}(D|\A|\poly(d))$. Therefore, the time complexity of Algorithm~\ref{alg:pac} is $\widetilde{\mathcal{O}}\left(\frac{HT|\A|\poly(d)}{\varepsilon^2}\right)$.

\section{Regret Minimization via Frozen Policy Iteration}\label{sec:regret}
In this section, we present the regret minimization version of Frozen Policy Iteration.
See Algorithm~\ref{alg:regret} for the formal description. 

\begin{algorithm}[htbp]
\caption{Frozen Policy Iteration-Regret Minimization (FPI-Regret)}\label{alg:regret}
\begin{algorithmic}[1]
\State Initialize $\D_h^{(l)}$ as empty list for all $l \ge 0, h \in [H]$
\For{$t$ = $1,\cdots,T$}
    \State \label{step:l_init} $l\gets\overline{L}$ \Comment{$\overline{L}$ is defined in Section~\ref{sec:constants}}
    \State get $s_1^{(t)}$
    \For{$h$ = $1,\cdots,H$}
        \If{all of $\mathfrak{I}_t^{(1)}(s_h^{(t)}),\cdots,\mathfrak{I}_t^{(l)}(s_h^{(t)})$ are true} \Comment{$\mathfrak{I}_t^{(l)}$ is as defined in \eqref{eq:acc_ind}}
            \State $l_h^{(t)}\gets l$
            \State take action $a_h^{(t)}\gets\pi_t^{(l)}(s_h^{(t)})$ \Comment{$\pi_t^{(l)}$ is as defined in \eqref{eq:pi_regret}}
        \Else
            \State $h_t\gets h$
            \State \label{step:change_l}$l\gets\min\left\{l\ge 1:\mathfrak{I}_t^{(l)}(s_h^{(t)})\text{ is not true}\right\}$
            \State $l_h^{(t)}\gets l$
            \State pick any $a\in\A_t^{(l)}(s_h^{(t)})$ such that $(s_h^{(t)},a)\notin\cover(\D_h^{(l)},2^{-l})$
            \State take action $a_h^{(t)}\gets a$ \Comment{$\A_t^{(l)}$ is as defined in \eqref{eq:A_regret}}
        \EndIf
        \State get $r_h^{(t)},s_{h+1}^{(t)}$
    \EndFor
    \If{$l<\overline{L}$}
        \State $\hat{q}_t\gets\sum_{h=h_t}^H r_h^{(t)}$
        \State $\D_{h_t}^{(l)}\text{.append}(s_{h_t}^{(t)},a_{h_t}^{(t)},\hat{q}_t)$
    \EndIf
\EndFor
\end{algorithmic}
\end{algorithm}

\paragraph{Datasets and Accuracy Levels.} Algorithm~\ref{alg:regret} follows similar high-level framework as Algorithm~\ref{alg:pac}. 
However, Algorithm~\ref{alg:pac} uses a fixed accuracy parameter $\varepsilon$, which is inadequate for achieving $\sqrt{T}$-type regret bound. 
Instead, Algorithm~\ref{alg:regret} uses multiple accuracy levels $1 \le l < \overline{L}$, where level $l$ corresponds to an instance of Algorithm~\ref{alg:pac}  with accuracy $\varepsilon=2^{-l}$ and $\overline{L}$ is a fixed constant defined later (Section~\ref{sec:constants}).
For each $l\ge 1$ and $h\in[H]$, we maintain a dataset $\D_h^{(l)}$ which is an ordered sequence consisting of $(s_{h,i}^{(l)},a_{h,i}^{(l)},q_{h,i}^{(l)})$. 
As in Section~\ref{sec:pac}, for every $t\ge 1$, let $\D_{t,h}^{(l)}$ be the snapshot of $\D_h^{(l)}$ before the $t$-th round. We also write $\phi_{h,i}^{(l)}=\phi(s_{h,i}^{(l)},a_{h,i}^{(l)})$ and 
\(
	\Sigma_{h,k}^{(l)}=\lambda I+\sum_{i=1}^k\phi_{h,i}^{(l)}{\phi_{h,i}^{(l)}}^\top
\).

\paragraph{Adjusting Accuracy Level.} At the beginning of each episode, the accuracy level $l$ is initialized to be $\overline{L}$, i.e, the level with highest accuracy. 
For each step $h$, if there is an action $a$ such that $(s_h^{(t)}, a)$ cannot be covered by $\D_{t,h}^{(l')}$ with accuracy $2^{-l'}$ for some $l' \le l $, then we would replace $l$ with $l'$ (Step~\ref{step:change_l}), aiming at a lower accuracy of $2^{-l'}$. 
Concretely, for each $s\in\S_h$, define the indicator
\begin{equation}\label{eq:acc_ind}
\mathfrak{I}_t^{(l)}(s)=\mathbb{I}\left\{(s,a)\in\cover(\D_{t,h}^{(l)},2^{-l})\text{ for all }a\in\A_t^{(l)}(s)\right\}, 
\end{equation}
where $\A_t^{(l)}(s)$ is a subset of actions that will be defined later in this section. Here, $\cover(\D_{t,h}^{(l)},2^{-l})$ follows the same definition as in~\eqref{eq:cover}. 
If all of $\mathfrak{I}_t^{(1)}(s_h^{(t)}),\cdots,\mathfrak{I}_t^{(l)}(s_h^{(t)})$ are true, we keep $l$ unchanged and follows a greedy policy $\pi_t^{(l)}$ for the purpose of exploitation. 
Otherwise, we replace $l$ with the smallest $l'$ so that $\mathfrak{I}_t^{(l')}(s_h^{(t)})$ is false, and take an exploratory action $a$. 
Once an episode finished, we add a data $(s_{h_t}^{(t)},a_{h_t}^{(t)},\hat{q}_t)$ into $\D_{h_t}^{(l)}$, where $h_t$ is the last level where we take an exploratory action, and $\hat{q}_t$ is the sum of rewards starting from step $h_t$ on the trajectory, and proceed to the next round. 

\paragraph{Exploration under Accuracy Level Constraints.} A subtlety in Algorithm~\ref{alg:regret} is that how the exploratory action is chosen. Unlike Algorithm~\ref{alg:pac} where any action $a$ satisfying $(s,a)\notin\cover(\D_{t, h},\varepsilon)$ could be chosen, here we also need to make sure that the suboptimality incurred by the chosen exploratory action is upper bounded by (roughly) $2^{-l}$. To this end, for the $t$-th round, for each state $s$ and accuracy level $l$, we define a subset of actions, $\A_t^{(l)}(s)$, which include all actions with suboptimality upper bounded by $2^{-l}$. To estimate the suboptimality of actions, we use the estimated $Q$-value at accuracy level $l - 1$.
Finally, when deciding whether to the freeze the policy of a state $s$ at the $t$-th round and accuracy level $l$, we only test whether $(s, a)$ could be covered by the dataset for those actions in $\A_t^{(l)}(s)$, instead of the whole action space $\A$. 

Formally, we define $\A_t^{(l)}$, $k_t^{(l)}$ and $Q_t^{(l)}$ inductively on $l$ as follows: 
\begin{align*}
	&\A_t^{(l)}(s)=\begin{cases}
		\A & l=1\\
		\left\{a\in\A_t^{(l-1)}(s):Q_t^{(l-1)}(s,a)\ge Q_t^{(l-1)}(s,\pi_t^{(l-1)}(s))-2H\Delta_{l-1}\right\} & l>1\\
	\end{cases};\numberthis\label{eq:A_regret}\\
	&k_t^{(l)}(s)=\min\left\{k\in\N:(s,a)\in\cover\left(\left\{(s_{h,i}^{(l)},a_{h,i}^{(l)})\right\}_{i=1}^k,2^{-l}\right)\text{ for all }a\in\A_t^{(l)}(s)\right\}\wedge\left|\D_{t,h}^{(l)}\right|;\\
	&Q_t^{(l)}(s,a)=\tilde{Q}_{k_t^{(l)}(s)}^{(l)}(s,a);\\
	&\pi_t^{(l)}(s)=\argmax_{a\in\A_t^{(l)}(s)}Q_t^{(l)}(s,a),\numberthis\label{eq:pi_regret}
\end{align*}
where $\Delta_l$ is defined in Section~\ref{sec:constants} and 
$\tilde{Q}_k^{(l)}(s,a)=\left\langle\phi(s,a),{\Sigma_{h,k}^{(l)}}^{-1}\sum_{i=1}^k\phi_{h,i}^{(l)}q_{h,i}^{(l)}\right\rangle$
is the least squares estimate by using only the first $k$ data in $\D_h^{(l)}$ as in Section~\ref{sec:pac}. 

The formal guarantee of Algorithm~\ref{alg:regret} is summarized as the following theorem. 
\begin{theorem}\label{thm:main_regret}
	With probability $1-\delta$, for any $T\ge 1$, the regret incurred by Algorithm~\ref{alg:regret} satisfies
	$\operatorname{Reg}(T)=\widetilde{\mathcal{O}}\left(\sqrt{d^2H^6T}+\sqrt{d}H^2T\kappa\right)$.
\end{theorem}
Theorem~\ref{thm:main_regret} is proved using a series of lemmas similar to those in Section~\ref{sec:pac}, though the proofs here are more complicated due to the use of multiple accuracy levels. See Section~\ref{sec:regret_proof} for the details. 

\paragraph{Time and Space Complexity.} In Algorithm~\ref{alg:regret}, at most one dataset will be updated at each round, so its space complexity is $\mathcal{O}(T\poly(d))$. 
Note that $l_h^{(t)}$ is upper bounded by $t$, so we can replace Step~\ref{step:l_init} with $l\gets\min\{t,\overline{L}\}$, which ensures that $\overline{L}$ does not affect the actual computational time. The time complexity of Algorithm~\ref{alg:regret} is $\widetilde{\mathcal{O}}(HT^2|\A|\poly(d))$.

\paragraph{Extensions.} Algorithm~\ref{alg:regret} also enjoys a Uniform-PAC guarantee of $\widetilde{\mathcal{O}}\left(\frac{d^2H^6}{\varepsilon^2}\right)$. See Section~\ref{sec:uniform_pac} for the details. Moreover, we further extend our results from linear function approximation to function classes with bounded eluder dimension. We provide bounds on $\operatorname{Reg}(T)$ of
$$
\mathcal{O}\left(\sqrt{H^6T\dim_E\left(\F,T^{-1}\right)\left(\log\frac{HT}{\delta}+\log N(\F,T^{-2},\Vert\cdot\Vert_\infty)\right)}\right).
$$
See Section~\ref{sec:eluder} for the details.

\section{Experiments}

We implement Algorithm~\ref{alg:pac} and conduct experiments on simple RL environments from the OpenAI gym. Our experiments are performed on CartPole-v1 and InvertedPendulum-v4.
We use the tile coding method \citep{sutton2018reinforcement} to produce a feature map in a continuous space. The number of tiles per dimension is set to 4. For InvertedPendulum-v4, we discretize the action space into 4 actions and set the number of steps per episode to 60. In the implementation, $\lambda$ is set to $10^{-3}$ and $\varepsilon$ is set to $1$. We maintain the covariance matrix $\Sigma_{h,k}$ to compute policy $\pi_t$. Since storing $\Sigma_{h,k}$ for all $h\in[H]$ and $k\le |\D_{t,h}|$ would lead to a huge memory cost, we only keep the last 20 $\Sigma_{h,k}$'s for each $h\in[H]$.

\makeatletter
\newcommand{\drawrewardplot}[3]{

  \edef\AR@upperpath{#2@upper}
  \edef\AR@lowerpath{#2@lower}
  \edef\AR@fillpath{#2@fill}

  \addplot[name path=\AR@upperpath, draw=none, forget plot] table[x=episode,y=upper]{\datatable};
  \addplot[name path=\AR@lowerpath, draw=none, forget plot] table[x=episode,y=lower]{\datatable};

  \addplot[forget plot, fill=#1!20, fill opacity=0.5, draw=none, forget plot]
    fill between[of={\AR@upperpath} and {\AR@lowerpath}];

  \addplot[thick, #1] table[x=episode,y=mean]{\datatable};

  \addlegendentry{#3}
}
\makeatother

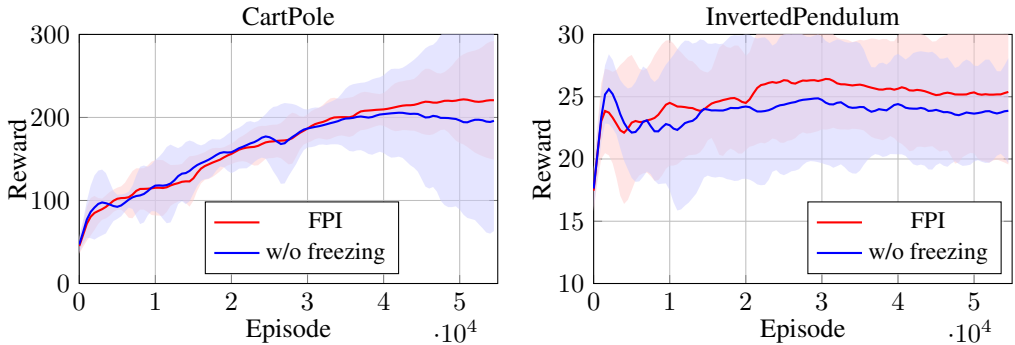
\begin{figure}[htbp]
\centering

\begin{tikzpicture}[scale=0.75]
\begin{axis}[
name=CartPole,
width=9cm, height=6cm,
xlabel=Episode,
ylabel=Reward,
xmin=0,xmax=55000,
ymin=0,ymax=300,
grid=major,
legend style={at={(0.4,0.05)},anchor=south west},
cycle list name=color list
]

\pgfplotstableread[col sep=comma]{data/CartPole.csv}\datatable
\drawrewardplot{red}{CartPole}{\small FPI}

\pgfplotstableread[col sep=comma]{data/CartPole-no_freeze.csv}\datatable
\drawrewardplot{blue}{CartPoleNoFreeze}{\small w/o freezing}

\end{axis}
\node[above] at (CartPole.north) {CartPole};
\end{tikzpicture}
\hfil
\begin{tikzpicture}[scale=0.75]
\begin{axis}[
name=InvertedPendulum,
width=9cm, height=6cm,
xlabel=Episode,
ylabel=Reward,
xmin=0,xmax=55000,
ymin=10,ymax=30,
grid=major,
legend style={at={(1.3,0.05)},anchor=south east},
cycle list name=color list
]

\pgfplotstableread[col sep=comma]{data/InvertedPendulum.csv}\datatable
\drawrewardplot{red}{InvertedPendulum}{\small FPI}

\pgfplotstableread[col sep=comma]{data/InvertedPendulum-no_freeze.csv}\datatable
\drawrewardplot{blue}{InvertedPendulumNoFreeze}{\small w/o freezing}

\end{axis}
\node[above] at (InvertedPendulum.north) {InvertedPendulum};
\end{tikzpicture}

\caption{Learning curves on CartPole-v1 and InvertedPendulum-v4. Results are averaged over 5 seeds. The shaded area captures a 95\% confidence interval around the average performance.}
\label{fig:rewards}
\end{figure}

To ablate the role of freezing, we also implement a version of Algorithm~\ref{alg:pac} without the freezing operation, i.e., using the whole dataset to estimate $Q$-values for each step $h\in[H]$. Figure~\ref{fig:rewards} shows that the freezing operation has indeed improved the performance of our algorithm.

\section{Discussions and Open Problems}

\paragraph{$H$ Factors in the Regret and Uniform-PAC Bounds.} The relatively high polynomial dependence on $H$ in the regret and Uniform-PAC bounds primarily arises from the need for exploration under multiple accuracy level constraints. Ensuring that the optimal action of each state is preserved across all accuracy levels introduces additional $H$ factors compared to the PAC analysis. 
Further improvements on the dependence of $H$ will be an interesting future work.

\paragraph{Reliance on Deterministic Transitions.} Applying our algorithm to MDPs with stochastic transitions is non-trivial, as the current analysis critically relies on the property that the state-action pairs collected in the datasets are well-explored. This property ensures that once a pair $(s,a)$ is added to the dataset, the subsequent trajectory remains within the high-confidence region, preserving the effectively on-policy nature of the data. 
Under stochastic dynamics, observing only one trajectory from a given pair $(s,a)$ does not guarantee that a sufficiently large or high-probability subset of trajectories starting from $(s,a)$ lies within the high-confidence region. 
Extending our algorithm to stochastic transition models hence remains an open problem.


\bibliography{iclr2026_conference}
\bibliographystyle{iclr2026_conference}

\appendix

\section*{The Use of Large Language Models}

In the preparation of this paper, we employed large language models (LLMs) as a general-purpose assistive tool for non-substantive tasks. Specifically, LLMs were used for:
\begin{itemize}
    \item Assisting in editing, proofreading, and formatting of textual content.
    \item Generating code snippets for creating figures and visualizations, which were then verified and customized by the authors.
\end{itemize}

The LLMs did not contribute to the core ideation, analysis, or substantive writing of the research. The authors take full responsibility for the entire content of this paper, including any portions influenced by LLMs, and affirm that the use of LLMs complies with academic integrity standards. LLMs are not considered authors or contributors to this work.
\section{Missing Proofs in Section~\ref{sec:pac}}\label{sec:pac_proof}

The proofs of the lemmas in Section~\ref{sec:pac} is almost the same with the ones in Section~\ref{sec:regret}.

See \cref{D_upper_bound,Q_same,event_high,Q_concentration,Q_acc,Q_near_optim} for the proofs of 
\cref{lem:pac_dataset_size_ub,lem:on_policy,lem:e,lem:tilde_Q,lem:Qt,lem:sub}, respectively.

\section{Missing Proofs in Section~\ref{sec:regret}}\label{sec:regret_proof}

\subsection{Notations}\label{sec:constants}
Given any $\delta\in(0,1)$, define constants
\begin{align*}
	\lambda&=H^{-1}\\
    \overline{L}&=\left\lfloor\ln\frac{\sqrt{H}}{\kappa}\right\rfloor+1\text{ (or $+\infty$ if $\kappa=0$)}\\
	D_l&=2d\cdot 2^{2l}\ln\left(1+\frac{2^{4l+2}}{\lambda^2}\right)&&=\Theta\left(dl\cdot 2^{2l}\right)\\
	\alpha_l&=\sqrt{2H\left(\frac{d}{2}\ln\left(1+\frac{D_l}{\lambda d}\right)+\ln\frac{2Hl^2}{\delta}\right)}+\sqrt{D_l}\kappa+\sqrt{\lambda d}H&&=\mathcal{O}\left(\sqrt{dHl\log\frac{H}{\delta}}+2^l\sqrt{dl}\kappa\right)\\
	\Delta_l&=\alpha_l\cdot 2^{-l}+\kappa&&=\mathcal{O}\left(2^{-l}\sqrt{dHl\log\frac{H}{\delta}}+\sqrt{dl}\kappa\right)
\end{align*}

For every $l\ge 1,h\in[H],i\ge 1$, let $t_{h,i}^{(l)}$ be the round when $(s_{h,i}^{(l)},a_{h,i}^{(l)},q_{h,i}^{(l)})$ is appended to $\D_h^{(l)}$. For every round $t\ge 1$, let $D_{t,h}^{(l)}$ be the size of $\D_{t,h}^{(l)}$ and $\xi_t$ be $\hat{q}_t-Q^{\pi_t^{(l)}}(s_{h_t}^{(t)},a_{h_t}^{(l)})$.

\subsection{Upper bound for the size of datasets}

\begin{lemma}\label{explore_bound}
	For any $l\ge 1,h\in[H],i\ge 1$, it holds that $\left\Vert\phi_{h,i}^{(l)}\right\Vert_{{\Sigma_{h,i-1}^{(l)}}^{-1}}\ge 2^{-l}$.
\end{lemma}

\begin{proof}
	Denote $t_{h,i}^{(l)}$ by $t$. From \cref{alg:regret}
	$$
	\left(s_h^{(t)},a_h^{(t)}\right)\notin\operatorname{Cover}\left(\D_{t,h}^{(l)},2^{-l}\right)
	$$
	so we have
	$$
	\left\Vert\phi_{h,i}^{(l)}\right\Vert_{{\Sigma_{h,i-1}^{(l)}}^{-1}}=\left\Vert\phi\left(s_h^{(t)},a_h^{(t)}\right)\right\Vert_{{\Sigma_{h,D_{t,h}^{(l)}}}^{-1}}\ge 2^{-l}
	$$
\end{proof}

\begin{lemma}\label{D_upper_bound}
	For any $t\ge 1,l\ge 1,h\in[H]$, it holds that $D_{t,h}^{(l)}\le D_l$.
\end{lemma}

\begin{proof}
	Denote $D_{t,h}^{(l)}$ by $D$. By \cref{explore_bound}, $2^l\left\Vert\phi_{h,i}^{(l)}\right\Vert_{{\Sigma_{h,i-1}^{(l)}}^{-1}}\ge 1$ for all $1\le i\le D$, so we have
	\begin{align*}
		D&\le\sum_{i=1}^D\min\left(1,2^{2l}\left\Vert\phi_{h,i}^{(l)}\right\Vert_{{\Sigma_{h,i-1}^{(l)}}^{-1}}^2\right)\\
		&\le 2^{2l}\sum_{i=1}^D\min\left(1,\left\Vert\phi_{h,i}^{(l)}\right\Vert_{{\Sigma_{h,i-1}^{(l)}}^{-1}}^2\right)\\
		&\le 2^{2l}\cdot 2\ln\left(\frac{\det\Sigma_{h,D}^{(l)}}{\det\Sigma_{h,0}^{(l)}}\right) \tag{elliptical potential lemma~\citep{lattimore2020bandit}}\\
		&\le 2d\cdot 2^{2l}\ln\left(1+\frac{D}{\lambda d}\right) \numberthis\label{D_intermediate_bound}\\
		&\le 2d\cdot 2^{2l}\sqrt{\frac{D}{\lambda d}} \tag{$\ln(1+x)\le\sqrt{x}$ for $x\ge 0$}
	\end{align*}
	Therefore, $D\le\frac{2^{4l+2}\cdot d}{\lambda}$. Plugging it into \eqref{D_intermediate_bound} gives the desired result.
\end{proof}

\subsection{Avoiding resampling}

\begin{lemma}\label{stay_same}
	Given $t\ge 1,l\ge 1,h\in[H],s\in\S_h$, if all of $\mathfrak{I}_t^{(1)}(s),\cdots,\mathfrak{I}_t^{(l)}(s)$ are true, then for any $t^\prime\ge t$,
	\begin{equation}\label{eq_stay_same}
	\begin{aligned}
		Q_{t^\prime}^{(l)}(s,a)&=Q_t^{(l)}(s,a),\forall a\in\A_{t^\prime}^{(l)}(s)\\
		\pi_{t^\prime}^{(l)}(s)&=\pi_t^{(l)}(s)\\
		\A_{t^\prime}^{(l+1)}(s)&=\A_t^{(l+1)}(s)
	\end{aligned}
	\end{equation}
\end{lemma}

\begin{proof}
	Assume by induction \eqref{eq_stay_same} holds for $l=l_0-1$. Since $\mathfrak{I}_t^{(l)}(s)$ is true, $k_{t^\prime}^{(l)}(s)=k_t^{(l)}(s)\le D_{t,h}^{(l)}$ for all $t^\prime\ge t,a\in\A_t^{(l)}(s)$. Note that $Q_{t^\prime}^{(l)}(s,a)$ depends only on $k_{t^\prime}^{(l)}(s)$; $\pi_{t^\prime}^{(l)}(s)$ depends only on $\A_{t^\prime}^{(l)}(s)$ and $Q_{t^\prime}^{(l)}(s,a)$; and $\A_{t^\prime}^{(l+1)}(s)$ depends only on $\A_{t^\prime}^{(l)}(s)$ and $Q_{t^\prime}^{(l)}(s,a)$; so \eqref{eq_stay_same} holds also for $l=l_0$.
\end{proof}

\begin{lemma}\label{Q_same}
	For any $l\ge 1,h\in[H],i\ge 1$, $Q^{\pi_t^{(l)}}(s_{h,i}^{(l)},a_{h,i}^{(l)})=Q^{\pi_{t^\prime}^{(l)}}(s_{h,i}^{(l)},a_{h,i}^{(l)})$ for all $t^\prime\ge t$, where $t=t_{h,i}^{(l)}$.
\end{lemma}

\begin{proof}
	From \cref{alg:regret} we have $s_h^{(t)}=s_{h,i}^{(l)},a_h^{(t)}=a_{h,i}^{(l)}$, and $\mathfrak{I}_t^{(l^\prime)}(s_{h^\prime}^{(t)})$ is true for all $1\le l^\prime\le l,h<h^\prime\le H$. By \cref{stay_same}, $\pi_{t^\prime}^{(l)}(s_{h^\prime}^{(t)})=\pi_t^{(l)}(s_{h^\prime}^{(t)})$ for all $h<h^\prime\le H$, so the trajectory starting from $(s_h^{(t)},a_h^{(t)})$ using policy $\pi_t^{(l)}$ is the same with the one using policy $\pi_{t^\prime}^{(l)}$.
\end{proof}

\subsection{Concentration bound for Q-values}

Let $\{\F_t\}_{t=0}^\infty$ be a filtration such that $\F_t$ is the $\sigma$-algebra generated by $(r_h^{(i)})_{1\le i\le t,h\in[H]}$. Then $\pi_t^{(l)}(s)$ is $\F_{t-1}$-measurable for all $t\ge 1,l\ge 1,s\in\S$. Denote $\E[\cdot|\F_t]$ by $\E_t[\cdot]$.

\begin{lemma}\label{xi_martingale_subgaussian}
	$\{\xi_t\}_{t=1}^\infty$ is a martingale difference sequence w.r.t. $\{\F_t\}_{t=0}^\infty$.
	Moreover, $\xi_t|\F_{t-1}$ is $\sqrt{H}$-subGaussian for all $t\ge 1$.
\end{lemma}

\begin{proof}
	Fix any $t\ge 1$. Denote $l_{h_t}^{(t)}$ by $l$. From \cref{alg:regret} we have $a_h^{(t)}=\pi_t^{(l)}(s_h^{(t)})$ for all $h_t<h\le H$, so
	$$
	\E_{t-1}\left[\left.\sum_{h=h_t}^H r_h^{(t)}\right|h_t\right]=Q^{\pi_t^{(l)}}\left(s_{h_t}^{(t)},a_{h_t}^{(t)}\right)
	$$
	By tower rule, we have $\E_{t-1}[\xi_t]=\E_{t-1}\left[\left.\sum_{h=h_t}^H r_h^{(t)}-Q^{\pi_t^{(l)}}(s_{h_t}^{(t)},a_{h_t}^{(t)})\right|h_t\right]=0$.
	Moreover, since $s_h^{(t)}$ and $a_h^{(t)}$ are $\F_{t-1}$ measurable for all $h\in[H]$ by the determinism of the MDP transitions, $\xi_t|\F_{t-1}$ is the sum of at most $H$ independent $1$-subGaussians, and is therefore $\sqrt{H}$-subGaussian.
\end{proof}

\begin{lemma}[Concentration of Self-Normalized Processes \citep{abbasi2011improved}]\label{concentration_bound}
	Let $\{\varepsilon_t\}_{t=1}^\infty$ be a martingale difference sequence w.r.t. $\{\F_t\}_{t=0}^\infty$. Let $\varepsilon_t|\F_{t-1}$ be $\sigma$-subGaussian. Let $\{\phi_t\}_{t=1}^\infty$ be an $\R^d$-valued stochastic process such that $\phi_t$ is $\F_{t-1}$-measurable. Let $\Lambda_0$ be a $d\times d$ positive definite matrix, and $\Lambda_t$ be $\Lambda_0+\sum_{i=1}^t\phi_i\phi_i^\top$. Then for any $\delta>0$, with probability at least $1-\delta$, we have for all $t\ge 1$,
	$$
	\left\Vert\sum_{i=1}^t\phi_i\varepsilon_i\right\Vert_{\Lambda_t^{-1}}^2\le 2\sigma^2\ln\left(\frac{\det(\Lambda_t)^{1/2}\det(\Lambda_0)^{-1/2}}{\delta}\right)
	$$
\end{lemma}

\begin{lemma}\label{event_high}
	Define event $\mathfrak{E}_\text{high}$ as
	$$
	\left\{\text{$\forall l\ge 1,h\in[H],k\ge 1$, $\left\Vert\sum_{i=1}^k\phi_{h,i}^{(l)}\xi_{t_{h,i}^{(l)}}\right\Vert_{{\Sigma_{h,k}^{(l)}}^{-1}}^2\le 2H\left(\frac{d}{2}\ln\left(1+\frac{k}{\lambda d}\right)+\ln\frac{2Hl^2}{\delta}\right)$}\right\}
	$$
	Then $\P[\mathfrak{E}_\text{high}]\ge 1-\delta$.
\end{lemma}

\begin{proof}
	For any $l\ge 1,h\in[H]$, $\{\xi_{t_{h,i}^{(l)}}\}_{i=1}^\infty$ is a subsequence of $\{\xi_t\}_{t=1}^\infty$. By \cref{xi_martingale_subgaussian,concentration_bound}, with probability at least $1-\frac{\delta}{2Hl^2}$, we have for all $k\ge 1$,
	\begin{align*}
		\left\Vert\sum_{i=1}^k\phi_{h,i}^{(l)}\xi_{t_{h,i}^{(l)}}\right\Vert_{{\Sigma_{h,k}^{(l)}}^{-1}}^2&\le 2H\ln\left(\frac{\det\left(\Sigma_{h,k}^{(l)}\right)^{1/2}\det\left(\Sigma_{h,0}^{(l)}\right)^{-1/2}}{\delta/(2Hl^2)}\right)\\
		&\le 2H\left(\frac{d}{2}\ln\frac{\operatorname{trace}\left(\Sigma_{h,k}^{(l)}\right)}{\lambda d}+\ln\frac{2Hl^2}{\delta}\right)\\
		&\le 2H\left(\frac{d}{2}\ln\left(1+\frac{k}{\lambda d}\right)+\ln\frac{2Hl^2}{\delta}\right)
	\end{align*}
	A union bound over $l$ and $h$ concludes the proof.
\end{proof}

\begin{lemma}[\citet{zanette2020learning}, Lemma 8]\label{elliptical_bias}
	Let $\{\delta_i\}_{i=1}^n$ be a real-valued sequence such that $|\delta_i|\le\kappa$, and $\{\phi_i\}_{i=1}^n$ be an $\R^d$-valued sequence. For $\Lambda=\lambda I+\sum_{i=1}^n\phi_i\phi_i^\top$ we have
	$$
	\left\Vert\sum_{i=1}^n\phi_i\delta_i\right\Vert_{\Lambda^{-1}}^2\le n\kappa^2
	$$
\end{lemma}

\begin{lemma}\label{Q_concentration}
	Under $\mathfrak{E}_\text{high}$, for any $t\ge 1,l\ge 1,h\in[H],0\le k\le D_{t,h}^{(l)},s\in\S_h,a\in\A$, it holds that
	$$
	\left|\tilde{Q}_k^{(l)}(s,a)-Q^{\pi_t^{(l)}}(s,a)\right|\le\alpha_l\Vert\phi(s,a)\Vert_{{\Sigma_{h,k}^{(l)}}^{-1}}+\kappa
	$$
\end{lemma}

\begin{proof}
	Let $\delta_i$ be $Q^{\pi_t^{(l)}}\left(s_{h,i}^{(l)},a_{h,i}^{(l)}\right)-\left\langle\phi_{h,i}^{(l)},\theta_h^{\pi_t^{(l)}}\right\rangle$. Then by assumption $|\delta_i|\le\kappa$.
	\begin{align*}
		&~~~\left|\tilde{Q}_k^{(l)}(s,a)-Q^{\pi_t^{(l)}}(s,a)\right|\\
		&\le\left|\tilde{Q}_k^{(l)}(s,a)-\left\langle\phi(s,a),\theta_h^{\pi_t^{(l)}}\right\rangle\right|+\left|\left\langle\phi(s,a),\theta_h^{\pi_t^{(l)}}\right\rangle-Q^{\pi_t^{(l)}}(s,a)\right|\\
		&\le\left|\tilde{Q}_k^{(l)}(s,a)-\left\langle\phi(s,a),\theta_h^{\pi_t^{(l)}}\right\rangle\right|+\kappa\\
		&~~~\tilde{Q}_k^{(l)}(s,a)-\left\langle\phi(s,a),\theta_h^{\pi_t^{(l)}}\right\rangle\\
		&=\left\langle\phi(s,a),{\Sigma_{h,k}^{(l)}}^{-1}\sum_{i=1}^k\phi_{h,i}^{(l)}q_{h,i}^{(l)}-\theta^{\pi_t^{(l)}}_h\right\rangle\\
		&=\left\langle\phi(s,a),{\Sigma_{h,k}^{(l)}}^{-1}\sum_{i=1}^k\phi_{h,i}^{(l)}\left(Q^{\pi_{t_{h,i}^{(l)}}^{(l)}}\left(s_{h,i}^{(l)},a_{h,i}^{(l)}\right)+\xi_{t_{h,i}^{(l)}}\right)-\theta^{\pi_t^{(l)}}_h\right\rangle\\
		&=\left\langle\phi(s,a),{\Sigma_{h,k}^{(l)}}^{-1}\sum_{i=1}^k\phi_{h,i}^{(l)}\left(Q^{\pi_t^{(l)}}\left(s_{h,i}^{(l)},a_{h,i}^{(l)}\right)+\xi_{t_{h,i}^{(l)}}\right)-\theta^{\pi_t^{(l)}}_h\right\rangle \tag{\cref{Q_same}}\\
		&=\left\langle\phi(s,a),{\Sigma_{h,k}^{(l)}}^{-1}\sum_{i=1}^k\phi_{h,i}^{(l)}\left({\phi_{h,i}^{(l)}}^\top\theta_h^{\pi_t^{(l)}}+\delta_i+\xi_{t_{h,i}^{(l)}}\right)-\theta^{\pi_t^{(l)}}_h\right\rangle\\
		&=\left\langle\phi(s,a),{\Sigma_{h,k}^{(l)}}^{-1}\sum_{i=1}^k\phi_{h,i}^{(l)}\left(\delta_i+\xi_{t_{h,i}^{(l)}}\right)-\lambda{\Sigma_{h,k}^{(l)}}^{-1}\theta_h^{\pi_t^{(l)}}\right\rangle\\
		&\le\Vert\phi(s,a)\Vert_{{\Sigma_{h,k}^{(l)}}^{-1}}\left(\left\Vert\sum_{i=1}^k\phi_{h,i}^{(l)}\xi_{t_{h,i}^{(l)}}\right\Vert_{{\Sigma_{h,k}^{(l)}}^{-1}}+\left\Vert\sum_{i=1}^k\phi_{h,i}^{(l)}\delta_i\right\Vert_{{\Sigma_{h,k}^{(l)}}^{-1}}+\left\Vert\lambda\theta_h^{\pi_t^{(l)}}\right\Vert_{{\Sigma_{h,k}^{(l)}}^{-1}}\right) \tag{Cauchy-Schwartz}\\
		&\le\Vert\phi(s,a)\Vert_{{\Sigma_{h,k}^{(l)}}^{-1}}\left(\sqrt{2H\left(\frac{d}{2}\ln\left(1+\frac{k}{\lambda d}\right)+\ln\frac{2Hl^2}{\delta}\right)}+\sqrt{k}\kappa+\sqrt{\lambda}\left\Vert\theta^{\pi_t^{(l)}}_h\right\Vert_2\right) \tag{\cref{concentration_bound,elliptical_bias}}\\
		&\le\alpha_l\Vert\phi(s,a)\Vert_{{\Sigma_{h,k}^{(l)}}^{-1}} \tag{\cref{assump:bound,D_upper_bound}}
	\end{align*}
\end{proof}

\begin{lemma}\label{Q_acc}
	Under $\mathfrak{E}_\text{high}$, given $t\ge 1,l\ge 1,h\in[H],s\in\S_h$, if $\mathfrak{I}_t^{(l)}(s)$ is true, then for any $a\in\A_t^{(l)}(s)$,
	$$
	\left|Q_t^{(l)}(s,a)-Q^{\pi_t^{(l)}}(s,a)\right|\le\Delta_l
	$$
\end{lemma}

\begin{proof}
	Denote $k_t^{(l)}(s)$ by $k$. By definition $k\le D_{t,h}^{(l)}$. Since $\mathfrak{I}_t^{(l)}(s)$ is true, $\left\Vert\phi(s,a)\right\Vert_{{\Sigma_{h,k}^{(l)}}^{-1}}\le 2^{-l}$ for all $a\in\A_t^{(l)}(s)$, so we have
	\begin{align*}
		&~~~\left|Q_t^{(l)}(s,a)-Q^{\pi_t^{(l)}}(s,a)\right|\\
		&=\left|\tilde{Q}_k^{(l)}(s,a)-Q^{\pi_t^{(l)}(s,a)}\right| \tag{by definition}\\
		&\le\alpha_l\Vert\phi(s,a)\Vert_{{\Sigma_{h,k}^{(l)}}^{-1}}+\kappa \tag{\cref{Q_concentration}}\\
		&\le\Delta_l
	\end{align*}
\end{proof}

\subsection{Regret analysis}

\begin{lemma}\label{elliptical_coef}
	Given $\varepsilon>0,v_0,v_1,\cdots,v_n\in\R^d$, let $\Lambda$ be $\lambda I+\sum_{i=1}^n v_iv_i^\top$. If $\Vert v_0\Vert_{\Lambda^{-1}}\le\varepsilon$, then there exist coefficients $c_1,\cdots,c_n\in\R$ such that $\sum_{i=1}^n c_i^2\le\varepsilon^2$ and $\left\Vert v_0-\sum_{i=1}^n c_iv_i\right\Vert_2\le\sqrt{\lambda}\varepsilon$.
\end{lemma}

\begin{proof}
	Let $u$ be $\Lambda^{-1}v_0$. Then $v_0=\Lambda u=\lambda u+\sum_{i=1}^n v_iv_i^\top u$. Pick $c_i=v_i^\top u$. We can verify that
	\begin{align*}
		\sum_{i=1}^n c_i^2&=\sum_{i=1}^n\left(v_i^\top u\right)^2\le\lambda u^\top u+\sum_{i=1}^n\left(v_i^\top u\right)^2=u^\top\Lambda u=\Vert v_0\Vert_{\Lambda^{-1}}^2\le\varepsilon^2\\
		\left\Vert v_0-\sum_{i=1}^n c_iv_i\right\Vert_2&=\Vert\lambda u\Vert_2\le\sqrt{\lambda}\Vert u\Vert_{\Lambda}\le\sqrt{\lambda}\varepsilon
	\end{align*}
\end{proof}

\begin{lemma}\label{Q_linear_dependence}
	For $\varepsilon>0,1\le h_1<h_2\le H$, given a policy $\pi_0$ and $n+1$ trajectories
	$$
	\left\{\left(s_{h_1}^{(i)},a_{h_1}^{(i)},\cdots,s_{h_2}^{(i)},a_{h_2}^{(i)}\right)\right\}_{i=0}^n
	$$
	such that $\pi_0(s_h^{(i)})=a_h^{(i)}$ for all $0\le i\le n,h_1+1\le h\le h_2$, if
	$$
	\left(s_{h_1}^{(0)},a_{h_1}^{(0)}\right)\in\operatorname{Cover}\left(\left\{\left(s_{h_1}^{(i)},a_{h_1}^{(i)}\right)\right\}_{i=1}^n,\varepsilon\right)
	$$
	and $s_{h_2}^{(0)}\neq s_{h_2}^{(i)}$ for all $1\le i\le n$, then for any policy $\pi$ and any $a_1,a_2\in\A$, it holds that
	$$
	\left|Q^\pi\left(s_{h_2}^{(0)},a_1\right)-Q^\pi\left(s_{h_2}^{(0)},a_2\right)\right|\le 2\left(\kappa+\sqrt{\lambda d}H\varepsilon+\sqrt{n}\kappa\varepsilon\right)
	$$
\end{lemma}

\begin{proof}
	Fix a policy $\pi$. Define policy $\overline{\pi}$ and $\overline{\pi}_a$ for $a\in\A$ as follows, where $\operatorname{stage}(s)=h$ if $s\in\S_h$, 
	\begin{align*}
		\overline{\pi}(s)&=\begin{cases}
			\pi_0(s) & \operatorname{stage}(s)\le h_2\\
			\pi(s) & \operatorname{stage}(s)>h_2
		\end{cases}\\
		\overline{\pi}_a(s)&=\begin{cases}
			a & s=s_{h_2}^{(0)}\\
			\overline{\pi}(s) & s\neq s_{h_2}^{(0)}
		\end{cases}
	\end{align*}
	Since $s_{h_2}^{(0)}\neq s_{h_2}^{(i)}$ for all $1\le i\le n$, we have $Q^{\overline{\pi}}(s_{h_1}^{(i)},a_{h_1}^{(i)})=Q^{\overline{\pi}_a}(s_{h_1}^{(i)},a_{h_1}^{(i)})$. Let $R$ be
	$$
	\E\left[r_{h_1}+\cdots+r_{h_2-1}\left|s_{h_1}=s_{h_1}^{(0)},a_{h_1}=a_{h_1}^{(0)},\pi_0\right.\right]
	$$
	Then $Q^{\overline{\pi}_a}(s_{h_1}^{(0)},a_{h_1}^{(0)})=R+Q^\pi(s_{h_2}^{(0)},a)$ for all $a\in\A$.

	Denote $\phi(s_{h_1}^{(i)},a_{h_1}^{(i)})$ by $\phi_i$. By \cref{elliptical_coef} there exist coefficients $c_1,\cdots,c_n\in\R$ such that $\sum_{i=1}^n c_i^2\le\varepsilon^2$ and
	$$
	\left\Vert\phi_0-\sum_{i=1}^n c_i\phi_i\right\Vert_2\le\sqrt{\lambda}\varepsilon
	$$
	Let $q_0$ be $\sum_{i=1}^n c_iQ^{\overline{\pi}}(s_{h_1}^{(i)},a_{h_1}^{(i)})-R$. Then we have for any $a\in\A$,
	\begin{align*}
		&~~~\left|Q^\pi\left(s_{h_2}^{(0)},a\right)-q_0\right|\\
		&=\left|Q^\pi\left(s_{h_2}^{(0)},a\right)+R-\sum_{i=1}^n c_iQ^{\overline{\pi}}\left(s_{h_1}^{(i)},a_{h_1}^{(i)}\right)\right|\\
		&=\left|Q^{\overline{\pi}_a}\left(s_{h_1}^{(0)},a_{h_1}^{(0)}\right)-\sum_{i=1}^n c_iQ^{\overline{\pi}_a}\left(s_{h_1}^{(i)},a_{h_1}^{(i)}\right)\right|\\
		&\le\left|Q^{\overline{\pi}_a}\left(s_{h_1}^{(0)},a_{h_1}^{(0)}\right)-\left\langle\phi_0,\theta_{h_1}^{\overline{\pi}_a}\right\rangle\right|+\left|\left\langle\phi_0-\sum_{i=1}^n c_i\phi_i,\theta_{h_1}^{\overline{\pi}_a}\right\rangle\right|+\sum_{i=1}^n c_i\left|\left\langle\phi_i,\theta_{h_1}^{\overline{\pi}_a}\right\rangle-Q^{\overline{\pi}_a}\left(s_{h_1}^{(i)},a_{h_1}^{(i)}\right)\right|\\
		&\le\kappa+\sqrt{\lambda}\varepsilon\left\Vert\theta_{h_1}^{\overline{\pi}_a}\right\Vert_2+\sum_{i=1}^n c_i\kappa\\
		&\le\kappa+\sqrt{\lambda d}H\varepsilon+\sqrt{n}\kappa\varepsilon
	\end{align*}
	which gives the desired result.
\end{proof}

\begin{lemma}\label{trajectory_explored}
	For any $l\ge 1,1\le h<h^\prime\le H,i\ge 1$, all of $\mathfrak{I}_t^{(1)}(s_{h^\prime}^{(t)}),\cdots,\mathfrak{I}_t^{(l)}(s_{h^\prime}^{(t)})$ are true, where $t=t_{h,i}^{(l)}$.
\end{lemma}

\begin{proof}
	From the algorithm we have $h=h_t$, so policy $\pi_t^{(l)}$ is used at stage $h+1,\cdots,H$, which means all of $\mathfrak{I}_t^{(1)}(s_{h^\prime}^{(t)}),\cdots,\mathfrak{I}_t^{(l)}(s_{h^\prime}^{(t)})$ are true.
\end{proof}

\begin{lemma}\label{Q_near_optim}
	Under $\mathfrak{E}_\text{high}$, given $t\ge 1,l\ge 1,h\in[H],s\in\S_h$, if all of $\mathfrak{I}_t^{(1)}(s),\cdots,\mathfrak{I}_t^{(l)}(s)$ are true, then for any $a\in\A_t^{(l)}(s)$,
	$$
	Q^{\pi_t^{(l)}}(s,a)\ge Q^{\pi^*}(s,a)-2(H-h)\Delta_l
	$$
	Furthermore, $\pi^*(s)\in\A_t^{(l+1)}(s)$.
\end{lemma}

\begin{proof}
	Assume by induction the result holds for $l=l_0-1$. Let the trajectory starting from $(s,a)$ using policy $\pi^*$ be $(s_h^*,a_h^*,s_{h+1}^*,a_{h+1}^*,\cdots,s_H^*,a_H^*)$, where $s_h^*=s,a_h^*=a$. For each $h_2=h+1,\cdots,H$,\begin{itemize}
		\item if $\mathfrak{I}_t^{(l)}(s_{h_2}^*)$ is true, then
		\begin{align*}
			&~~~Q^{\pi_t^{(l)}}\left(s_{h_2}^*,\pi_t^{(l)}\left(s_{h_2}^*\right)\right)\\
			&\ge Q_t^{(l)}\left(s_{h_2}^*,\pi_t^{(l)}\left(s_{h_2}^*\right)\right)-\Delta_l \tag{\cref{Q_acc}}\\
			&\ge Q_t^{(l)}\left(s_{h_2}^*,\pi^*\left(s_{h_2}^*\right)\right)-\Delta_l \tag{by the definition of $\pi_t^{(l)}$}\\
			&\ge Q^{\pi_t^{(l)}}\left(s_{h_2}^*,\pi^*\left(s_{h_2}^*\right)\right)-2\Delta_l \tag{\cref{Q_acc}}
		\end{align*}
		\item otherwise, let $h_1$ be $\max\left\{h^\prime<h_2:\text{all of $\mathfrak{I}_t^{(1)}(s_{h^\prime}^*),\cdots,\mathfrak{I}_t^{(l)}(s_{h^\prime}^*)$ are true}\right\}$, which always exists since $h$ satisfies the condition. Define policy
		$$
		\pi_0(s)=\begin{cases}
			\pi_t^{(l)}(s) & \text{all of $\mathfrak{I}_t^{(1)}(s),\cdots,\mathfrak{I}_t^{(l)}(s)$ are true}\\
			\pi^*(s) & \text{otherwise}
		\end{cases}
		$$
		By \cref{stay_same,trajectory_explored} we have that for any $1\le i\le D_{t,h}^{(l)}$, the trajectory starting from $(s_{h,i}^{(l)},a_{h,i}^{(l)})$ using policy $\pi_t^{(l)}$ is the same with the one using policy $\pi_0$. Starting from $(s_{h_1}^*,a_{h_1}^*)$ using policy $\pi_0$, the trajectory will reach $s_{h_2}^*$, which is not on any of the trajectories starting from $(s_{h,i}^{(l)},a_{h,i}^{(l)})$ because $\mathfrak{I}_t^{(l)}(s_{h_2}^*)$ is not true.

		Using the induction hypothesis we have $a_{h_1}^*\in\A_t^{(l)}(s_{h_1}^*)$, so $\left\Vert\phi(s_{h_1}^*,a_{h_1}^*)\right\Vert_{{\Sigma_{h,D_{t,h}^{(l)}}}^{-1}}\le 2^{-l}$. By \cref{Q_linear_dependence}, we have
		\begin{align*}
			&~~~Q^{\pi_t^{(l)}}\left(s_{h_2}^*,\pi_t^{(l)}\left(s_{h_2}^*\right)\right)\\
			&\ge Q^{\pi_t^{(l)}}\left(s_{h_2}^*,\pi^*\left(s_{h_2}^*\right)\right)-2\left(\kappa+\left(\sqrt{\lambda d}H+\sqrt{D_{t,h_1}^{(l)}}\kappa\right)\cdot 2^{-l}\right)\\
			&\ge Q^{\pi_t^{(l)}}\left(s_{h_2}^*,\pi^*\left(s_{h_2}^*\right)\right)-2\Delta_l
		\end{align*}
	\end{itemize}

	Therefore, we conclude that $Q^{\pi_t^{(l)}}(s_{h^\prime}^*,\pi_t^{(l)}(s_{h^\prime}^*))\ge Q^{\pi_t^{(l)}}(s_{h^\prime}^*,\pi^*(s_{h^\prime}^*))-2\Delta_l$ for all $h<h^\prime\le H$. By performance difference lemma, we have
	\begin{align*}
		&~~~Q^{\pi_t^{(l)}}(s,a)-Q^{\pi^*}(s,a)\\
		&=\sum_{h^\prime=h+1}^H\left(Q^{\pi_t^{(l)}}(s_{h^\prime}^*,\pi_t^{(l)}\left(s_{h^\prime}^*\right))-Q^{\pi_t^{(l)}}(s_{h^\prime}^*,\pi^*(s_{h^\prime}^*))\right)\\
		&\ge 2(H-h)\Delta_l
	\end{align*}

	Furthermore,
	\begin{align*}
		&~~~Q_t^{(l)}\left(s,\pi^*(s)\right)\\
		&\ge Q^{\pi_t^{(l)}}\left(s,\pi^*(s)\right)-\Delta_l \tag{\cref{Q_acc}}\\
		&\ge Q^{\pi^*}\left(s,\pi^*(s)\right)-(2(H-h)+1)\Delta_l\\
		&\ge Q^{\pi_t^{(l)}}\left(s,\pi_t^{(l)}(s)\right)-(2(H-h)+1)\Delta_l \tag{by the optimality of $\pi^*$}\\
		&\ge Q_t^{(l)}\left(s,\pi_t^{(l)}(s)\right)-2(H-h+1)\Delta_l \tag{\cref{Q_acc}}
	\end{align*}
	so $\pi^*(s)\in\A_t^{(l+1)}(s)$, which concludes the induction.
\end{proof}

\begin{lemma}\label{suboptim_gap}
	Under $\mathfrak{E}_\text{high}$, for any $t\ge 1,h\in[H]$, it holds that
	$$
	\E_{t-1}\left[\sum_{i=h}^H r_i^{(t)}\right]\ge V^{\pi^*}\left(s_h^{(t)}\right)-4H\sum_{i=h}^H\Delta_{l_i^{(t)}-1}
	$$
\end{lemma}

\begin{proof}
	We use induction on $h$ from $H$ to $1$.
	\begin{align*}
		&~~~\E_{t-1}\left[\sum_{i=h}^H r_i^{(t)}\right]\\
		&=\E_{t-1}\left[r_h^{(t)}\right]+\E_{t-1}\left[\sum_{i=h+1}^H r_i^{(t)}\right]\\
		&\ge\E_{t-1}\left[r_h^{(t)}\right]+V^{\pi^*}\left(s_{h+1}^{(t)}\right)-4H\sum_{i=h+1}^H\Delta_{l_i^{(t)}-1} & \text{(by induction)}\\
		&=Q^{\pi^*}\left(s_h^{(t)},a_h^{(t)}\right)-4H\sum_{i=h+1}^H\Delta_{l_i^{(t)}-1}\\
		&~~~Q^{\pi^*}\left(s_h^{(t)},a_h^{(t)}\right)\\
		&\ge Q^{\pi_t^{\left(l_h^{(t)}-1\right)}}\left(s_h^{(t)},a_h^{(t)}\right)\\
		&\ge Q_t^{\left(l_h^{(t)}-1\right)}\left(s_h^{(t)},a_h^{(t)}\right)-\Delta_{l_h^{(t)}-1} & \text{(\cref{Q_acc})}\\
		&\ge Q_t^{\left(l_h^{(t)}-1\right)}\left(s_h^{(t)},\pi_t^{\left(l_h^{(t)}-1\right)}\left(s_h^{(t)}\right)\right)-(2H+1)\Delta_{l_h^{(t)}-1} & \text{(by the definition of $\A_t^{\left(l_h^{(t)}\right)}$)}\\
		&\ge Q_t^{\left(l_h^{(t)}-1\right)}\left(s_h^{(t)},\pi^*\left(s_h^{(t)}\right)\right)-(2H+1)\Delta_{l_h^{(t)}-1} & \text{(by the definition of $\pi_t^{\left(l_h^{(t)}-1\right)}$)}\\
		&\ge Q^{\pi_t^{\left(l_h^{(t)}-1\right)}}\left(s_h^{(t)},\pi^*\left(s_h^{(t)}\right)\right)-(2H+2)\Delta_{l_h^{(t)}-1} & \text{(\cref{Q_acc})}\\
		&\ge Q^{\pi^*}\left(s_h^{(t)},\pi^*\left(s_h^{(t)}\right)\right)-4H\Delta_{l_h^{(t)}-1} & \text{(\cref{Q_near_optim})}\\
		&=V^{\pi^*}\left(s_h^{(t)}\right)-4H\Delta_{l_h^{(t)}-1}
	\end{align*}
\end{proof}

\begin{lemma}\label{bound_big_l}
	For any $1\le k<\overline{L}$, it holds that
	$$
	\sum_{t=1}^\infty\mathds{1}\left\{\exists h\in[H]\text{ s.t. }l_h^{(t)}\le k\right\}\le H\sum_{l=1}^k D_l
	$$
\end{lemma}

\begin{proof}
	For any $t\ge 1$ such that $l_{h_0}^{(t)}\le k$ for some $h_0\in[H]$, the size of exactly one $\D_h^{(l)}$ will increase by 1 after the $t$-the episode for some $l\le k$ and $h\in[H]$. Then by \cref{D_upper_bound} we have for any $T\ge 1$,
	$$
	\sum_{t=1}^T\mathds{1}\left\{\exists h\in[H]\text{ s.t. }l_h^{(t)}\le k\right\}\le\sum_{l=1}^k\sum_{h=1}^H D_{T,h}^{(l)}\le H\sum_{l=1}^k D_l
	$$
\end{proof}

\begin{lemma}\label{bound_l}
	For any $h\in[H]$, it holds that $\sum_{t=1}^T 2^{-l_h^{(t)}}=\widetilde{\mathcal{O}}\left(\sqrt{dHT}+\frac{T\kappa}{\sqrt{H}}\right)$.
\end{lemma}

\begin{proof}
	Let $L$ be $\left\lceil\frac{1}{2}\log_2\frac{T}{2Hd}\right\rceil$ so that $HD_L\ge T$. By \cref{bound_big_l} we have for any $1\le k<\overline{L}$,
	$$
	\sum_{t=1}^{T}\mathds{1}\left\{l_h^{(t)}\le k\right\}\le\min\left\{T,H\sum_{l=1}^k D_l\right\}\le H\min\left\{D_L,\sum_{l=1}^k D_l\right\}\le H\sum_{l=1}^{\min\{k,L\}}D_l
	$$
	Therefore,
	\begin{align*}
		\sum_{t=1}^T 2^{-l_h^{(t)}}&\le\sum_{t=1}^T\left(2^{-\overline{L}}+\sum_{k=1}^{\overline{L}-1}\left(2^{-k}-2^{-(k+1)}\right)\mathds{1}\left\{l_h^{(t)}\le k\right\}\right)\\
		&\le\frac{T\kappa}{\sqrt{H}}+H\sum_{l=1}^L D_l\cdot 2^{-l}\\
		&=\widetilde{\mathcal{O}}\left(\frac{T\kappa}{\sqrt{H}}+\sqrt{dHT}\right)
	\end{align*}
\end{proof}

\begin{proof}[Proof of Theorem~\ref{thm:main_regret}]
	By \cref{event_high}, $\mathfrak{E}_\text{high}$ holds with probability at least $1-\delta$, so we have
	\begin{align*}
		\operatorname{Reg}(T)&\le 4H\sum_{h=1}^H\sum_{t=1}^T\Delta_{l_h^{(t)}-1} \tag{\cref{suboptim_gap}}\\
		&\le H\sum_{h=1}^H\sum_{t=1}^T\widetilde{\mathcal{O}}\left(\sqrt{dH}\left(2^{-l_h^{(t)}}+\frac{\kappa}{\sqrt{H}}\right)\right)\\
		&\le H\sum_{h=1}^H\widetilde{\mathcal{O}}\left(\sqrt{dH}\left(\sqrt{dHT}+\frac{T\kappa}{\sqrt{H}}\right)\right) \tag{\cref{bound_l}}\\
		&=\widetilde{\mathcal{O}}\left(\sqrt{d^2H^6T}+\sqrt{d}H^2T\kappa\right)
	\end{align*}
\end{proof}

\subsection{Uniform-PAC}\label{sec:uniform_pac}

\begin{theorem}\label{thm:uniform_pac}
	Let $N(\varepsilon)$ be the number of episodes whose suboptimality gap is greater than $\varepsilon$. There exists a constant $C$ such that with probability at least $1-\delta$, it holds for all $\varepsilon>0$ that
	$$
	N\left(\varepsilon+C\sqrt{d}H^2\kappa\log\frac{H}{\kappa}\right)=\widetilde{\mathcal{O}}\left(\frac{d^2H^6}{\varepsilon^2}\right)
	$$
\end{theorem}

\begin{proof}
	Let $k$ be $\min\left\{\overline{L}-1,\log_2(\frac{4H^2}{\varepsilon}\sqrt{dH})+\Theta\left(\log\log\frac{H}{\delta}\right)\right\}$ so that there exists a constant $C$ such that $\Delta_l\le\frac{\varepsilon}{4H^2}+\frac{C}{4}\sqrt{d}\kappa\log\frac{H}{\kappa}$ for all $k\le l<\overline{L}$. With probability at least $1-\delta$, $\mathfrak{E}_\text{high}$ is true, so we have
	\begin{align*}
		&~~~N\left(\varepsilon+C\sqrt{d}H^2\kappa\log\frac{H}{\kappa}\right)\\
		&\le\sum_{t=1}^\infty\mathds{1}\left\{4H\sum_{h=1}^H\Delta_{l_h^{(t)}-1}>\varepsilon+C\sqrt{d}H^2\kappa\log\frac{H}{\kappa}\right\} \tag{\cref{suboptim_gap}}\\
		&\le\sum_{t=1}^\infty\mathds{1}\left\{\exists h\in[H]\text{ s.t. }\Delta_{l_h^{(t)}-1}>\frac{\varepsilon}{4H^2}+\frac{C}{4}\sqrt{d}\kappa\log\frac{H}{\kappa}\right\}\\
		&\le\sum_{t=1}^\infty\mathds{1}\left\{\exists h\in[H]\text{ s.t. }l_h^{(t)}\le k\right\} \tag{by the definition of $k$}\\
		&\le H\sum_{l=1}^k D_l \tag{\cref{bound_big_l}}\\
		&=\widetilde{\mathcal{O}}\left(\frac{d^2H^6}{\varepsilon^2}\right)
	\end{align*}
\end{proof}

\section{Functoin class with bounded eluder dimension}\label{sec:eluder}

We extend the $Q$-function into a more general function class. Let $\F=\{f_\rho:\S\times\A\to[0,H]|\rho\in\Theta\}$ be a real-valued function class.

\begin{assumption}
	For any policy $\pi$ there exist $\theta_1^{\pi},\cdots,\theta_H^{\pi}\in\Theta$ such that $f_{\theta_h^{\pi}}(s,a)=Q^{\pi}(s,a)$ for all $s\in\S_h,a\in\A$.
\end{assumption}

\begin{definition}
	A state-action pair $(s,a)\in\S\times\A$ is $\varepsilon$-\emph{dependent} on $\{(s_i,a_i)\}_{i=1}^n\subseteq\S\times\A$ w.r.t. $\F$ if any pair of functions $f_1,f_2\in\F$ satisfying $\sum_{i=1}^n(f_1(s_i,a_i)-f_2(s_i,a_i))^2\le\varepsilon^2$ also satisfies $|f_1(s,a)-f_2(s,a)|\le\varepsilon$. Further, $(s,a)$ is $\varepsilon$-independent of $\{(s_i,a_i)\}_{i=1}^n$ w.r.t. $\F$ if $(s,a)$ is not $\varepsilon$-dependent on $\{(s_i,a_i)\}_{i=1}^n$.
\end{definition}

\begin{definition}
	The $\varepsilon$-\emph{eluder dimension} $\dim_E(F, \varepsilon)$ is the length $d$ of the longest sequence of elements in $\S\times\A$ such that, for some $\varepsilon^\prime\ge\varepsilon$, every element is $\varepsilon^\prime$-independent of its predecessors.
\end{definition}

Define the width of a subset $\tilde{\F}\subseteq\F$ at $(s,a)\in\S\times\A$ as
$$
w_{\tilde{\F}}(s,a)=\sup_{\underline{f},\overline{f}\in\tilde{\F}}\left(\overline{f}(s,a)-\underline{f}(s,a)\right).
$$

For a dataset $\D=\left\{(s_i,a_i)\right\}_{i=1}^n$ and $\tilde{\F}\subseteq\F$, define
$$
\operatorname{Cover}(\D,\tilde{\F},\varepsilon)=\left\{(s,a)\in\S\times\A:\text{$(s,a)$ is $\varepsilon$-dependent on $\left\{(s_i,a_i)\right\}_{i=1}^n$ w.r.t. $\F$, and $w_{\tilde{\F}}(s,a)\le\varepsilon$}\right\}.
$$
Let $\beta_t^*(\F,\delta,\alpha)$ be
$$
8H\ln\frac{N\left(\F,\alpha,\Vert\cdot\Vert_\infty\right)}{\delta}+2\alpha t\left(8H+\sqrt{8H\ln\frac{4t^2}{\delta}}\right).
$$
where $N(\F,\alpha,\Vert\cdot\Vert_\infty)$ denotes the $\alpha$-covering number of $\F$ in the sup-norm $\Vert\cdot\Vert_\infty$.

Define least squares estimates and confidence sets as
\begin{align*}
	\hat{f}_{h,k}^{(l)}&=\argmin_{f\in\F}\sum_{i=1}^k\left(f(s_{h,i}^{(l)},a_{h,i}^{(l)})-q_{h,i}^{(l)}\right)^2\\
	\F_{h,k}^{(l)}&=\left\{f\in\F:\sum_{i=1}^k\left(f(s_{h,i}^{(l)},a_{h,i}^{(l)})-\hat{f}_{h,k}^{(l)}(s_{h,i}^{(l)},a_{h,i}^{(l)})\right)^2\le\beta_k^*\left(\F,\frac{\delta}{4Hl^2},\frac{1}{T^2}\right)\right\}
\end{align*}

We use the same procedure as Algorithm~\ref{alg:regret} while for every $h\in[H],s\in\S_h,a\in\A$, re-define
\begin{align*}
	\mathfrak{I}_t^{(l)}(s)&=\mathbb{I}\left\{\exists 1\le k\le D_{t,h}^{(l)}\text{ s.t. $(s,a)\in\operatorname{Cover}\left(\left\{(s_{h,i}^{(l)},a_{h,i}^{(l)})\right\}_{i=1}^k,\F_{h,k}^{(l)},2^{-l}\right)$ for all $a\in\A_t^{(l)}(s)$}\right\}\\
	Q_t^{(l)}(s,a)&=\hat{f}_{h,k_t^{(l)}(s)}(s,a)
\end{align*}
and re-define constants
\begin{align*}
	D_l&=\Theta\left(2^{2l}H\dim_E\left(\F,T^{-1}\right)\left(\log\frac{HT}{\delta}+\log N(\F,T^{-2},\Vert\cdot\Vert_\infty)\right)\right)\\
	\Delta_l&=2^{-l}\\
	\overline{L}&=\left\lfloor\log_2 T\right\rfloor+1
\end{align*}

\begin{lemma}
	For any $t\ge 1,1\le l<\overline{L},h\in[H]$, it holds that $D_{t,h}^{(l)}\le D_l$.
\end{lemma}

\begin{proof}
	Denote $D_{t,h}^{(l)}$ by $D$.
	\begin{align*}
		D&\le\sum_{i=1}^D\mathds{1}\left\{w_{\F_{h,i}^{(l)}}\left(s_{h,i}^{(l)},a_{h,i}^{(l)}\right)>2^{-l}\right\}+\dim_E\left(\F,2^{-l}\right)\\
		&\le\left(1+\frac{4\beta_D^*\left(\F,\frac{\delta}{4Hl^2},\frac{1}{T^2}\right)}{2^{-2l}}\right)\dim_E\left(\F,2^{-l}\right)+\dim_E\left(\F,2^{-l}\right) \tag{Proposition 3 from \citet{russo2013eluder}}\\
		&\le\mathcal{O}\left(2^{2l}\left(H\log\frac{Hl\cdot N\left(\F,T^{-2},\Vert\cdot\Vert_\infty\right)}{\delta}+\frac{H}{T}\sqrt{\log\frac{HT}{\delta}}\right)\right)\dim_E\left(\F,T^{-1}\right) \tag{$l<\overline{L}$}\\
		&\le D_l
	\end{align*}
\end{proof}

\begin{lemma}\label{eluder_event_high}
	Define event $\mathfrak{E}_\text{high}$ as
	$$
	\left\{\forall l\ge 1,h\in[H],k\ge 1,f_{\theta_h^{\pi_t^{(l)}}}\in\F_{h,k}^{(l)}\right\}
	$$
	Then $\P[\mathfrak{E}_\text{high}]\ge 1-\delta$.
\end{lemma}

\begin{proof}
	For every $l\ge 1,h\in[H]$, by \cref{xi_martingale_subgaussian} and Proposition 2 from \citet{russo2013eluder}, with probability at least $1-\frac{\delta}{2Hl^2}$, we have
	$$
	f_{\theta_h^{\pi_t^{(l)}}}\in\F_{h,k}^{(l)}
	$$
	for all $k\ge 1$. A union bound over $l$ and $h$ concludes the proof.
\end{proof}

\begin{lemma}\label{eluder_Q_acc}
	Under $\mathfrak{E}_\text{high}$, given $t\ge 1,l\ge 1,h\in[H],s\in\S_h$, if $\mathfrak{I}_t^{(l)}(s)$ is true, then for any $a\in\A_t^{(l)}(s)$,
	$$
	\left|Q_t^{(l)}(s,a)-Q^{\pi_t^{(l)}}(s,a)\right|\le 2^{-l}
	$$
\end{lemma}

\begin{proof}
	Denote $k_t^{(l)}(s,a)$ by $k$. Since $\mathfrak{I}_t^{(l)}(s)$ is true, $w_{\F_{h,k}^{(l)}}(s,a)\le 2^{-l}$ for all $a\in\A_t^{(l)}(s)$, so we have
	\begin{align*}
		&~~~\left|Q_t^{(l)}(s,a)-Q^{\pi_t^{(l)}}(s,a)\right|\\
		&=\left|\hat{f}_{h,k}^{(l)}(s,a)-f_{\theta_h^{\pi_t^{(l)}}}(s,a)\right|\\
		&\le w_{\F_{h,k}^{(l)}}(s,a)\le 2^{-l}
	\end{align*}
\end{proof}

\begin{lemma}\label{eluder_dependence}
	For $\varepsilon>0,\beta>0,1\le h_1<h_2\le H$, given a policy $\pi_0$, and $n+1$ trajectories
	$$
	\left\{\left(s_{h_1}^{(i)},a_{h_1}^{(i)},\cdots,s_{h_2}^{(i)},a_{h_2}^{(i)}\right)\right\}_{i=0}^n
	$$
	such that $\pi_0(s_h^{(i)})=a_h^{(i)}$ for all $0\le i\le n,h_1+1\le h\le h_2$, if $(s_{h_1}^{(0)},a_{h_1}^{(0)})$ is $\varepsilon$-dependent on $\left\{(s_{h_1}^{(i)},a_{h_1}^{(i)})\right\}_{i=1}^n$, and $s_{h_2}^{(0)}\neq s_{h_2}^{(i)}$ for all $1\le i\le n$, then for any policy $\pi$ and any $a_1,a_2\in\A$, it holds that
	$$
	\left|Q^\pi\left(s_{h_2}^{(0)},a_1\right)-Q^\pi\left(s_{h_2}^{(0)},a_2\right)\right|\le 2\varepsilon
	$$
\end{lemma}

\begin{proof}
	Fix a policy $\pi$. Define policy $\overline{\pi}$ and $\overline{\pi}_a$ for $a\in\A$ as follows, where $\operatorname{stage}(s)=h$ if $s\in\S_h$, 
	\begin{align*}
		\overline{\pi}(s)&=\begin{cases}
			\pi_0(s) & \operatorname{stage}(s)\le h_2\\
			\pi(s) & \operatorname{stage}(s)>h_2
		\end{cases}\\
		\overline{\pi}_a(s)&=\begin{cases}
			a & s=s_{h_2}^{(0)}\\
			\overline{\pi}(s) & s\neq s_{h_2}^{(0)}
		\end{cases}
	\end{align*}
	Since $s_{h_2}^{(0)}\neq s_{h_2}^{(i)}$ for all $1\le i\le n$, we have $Q^{\overline{\pi}}(s_{h_1}^{(i)},a_{h_1}^{(i)})=Q^{\overline{\pi}_a}(s_{h_1}^{(i)},a_{h_1}^{(i)})$. Let $R$ be
	$$
	\E\left[r_{h_1}+\cdots+r_{h_2-1}\left|s_{h_1}=s_{h_1}^{(0)},a_{h_1}=a_{h_1}^{(0)},\pi_0\right.\right]
	$$
	Then $Q^{\overline{\pi}_a}(s_{h_1}^{(0)},a_{h_1}^{(0)})=R+Q^\pi(s_{h_2}^{(0)},a)$ for all $a\in\A$.

	Since $\sum_{i=1}^n\left(f_{\theta_{h_1}^{\overline{\pi}}}(s_{h_1}^{(i)},a_{h_1}^{(i)})-f_{\theta_{h_1}^{\overline{\pi}_a}}(s_{h_1}^{(i)},a_{h_1}^{(i)})\right)^2=0\le\varepsilon^2$, we have
	$$
	\left|f_{\theta_{h_1}^{\overline{\pi}}}\left(s_{h_1}^{(0)},a_{h_1}^{(0)}\right)-f_{\theta_{h_1}^{\overline{\pi}_a}}\left(s_{h_1}^{(0)},a_{h_1}^{(0)}\right)\right|\le\varepsilon
	$$
	by the definition of $\varepsilon$-dependence. Let $q_0$ be $Q^{\overline{\pi}}(s_{h_1}^{(0)},a_{h_1}^{(0)})-R$. Then we have for any $a\in\A$,
	\begin{align*}
		&~~~\left|Q^{\pi}\left(s_{h_2}^{(0)},a\right)-q_0\right|\\
		&=\left|Q^{\pi}\left(s_{h_2}^{(0)},a\right)+R-Q^{\overline{\pi}}\left(s_{h_1}^{(0)},a_{h_1}^{(0)}\right)\right|\\
		&=\left|Q^{\overline{\pi}_a}\left(s_{h_1}^{(0)},a_{h_1}^{(0)}\right)-Q^{\overline{\pi}}\left(s_{h_1}^{(0)},a_{h_1}^{(0)}\right)\right|\\
		&=\left|f_{\theta_{h_1}^{\overline{\pi}_a}}\left(s_{h_1}^{(0)},a_{h_1}^{(0)}\right)-f_{\theta_{h_1}^{\overline{\pi}}}\left(s_{h_1}^{(0)},a_{h_1}^{(0)}\right)\right|\\
		&\le\varepsilon
	\end{align*}
	which gives the desired result.
\end{proof}

By \cref{eluder_Q_acc,eluder_dependence}, \cref{Q_near_optim,suboptim_gap} still hold.

\begin{lemma}\label{eluder_bound_l}
	For any $h\in[H]$, it holds that
	$$
	\sum_{t=1}^T 2^{-l_h^{(t)}}=\mathcal{O}\left(\sqrt{H^2T\dim_E\left(\F,T^{-1}\right)\left(\log\frac{HT}{\delta}+\log N(\F,T^{-2},\Vert\cdot\Vert_\infty)\right)}\right)
	$$
\end{lemma}

\begin{proof}
	Let $L$ be
	$$
	\frac{1}{2}\log_2\frac{T}{H^2\dim_E\left(\F,T^{-1}\right)\left(\log\frac{HT}{\delta}+\log N(\F,T^{-2},\Vert\cdot\Vert_\infty)\right)}+\Theta(1)
	$$
	so that $HD_L\ge T$. By \cref{bound_big_l} we have for any $1\le k<\overline{L}$,
	$$
	\sum_{t=1}^{T}\mathbb{I}\left[l_h^{(t)}\le k\right]\le\min\left\{T,H\sum_{l=1}^k D_l\right\}\le H\min\left\{D_L,\sum_{l=1}^k D_l\right\}\le H\sum_{l=1}^{\min\{k,L\}}D_l
	$$
	Therefore,
	\begin{align*}
		\sum_{t=1}^T 2^{-l_h^{(t)}}&\le\sum_{t=1}^T\left(2^{-\overline{L}}+\sum_{k=1}^{\overline{L}-1}\left(2^{-k}-2^{-(k+1)}\right)\mathds{1}\left\{l_h^{(t)}\le k\right\}\right)\\
		&\le\sqrt{T}+H\sum_{l=1}^L D_l\cdot 2^{-l}\\
		&=\mathcal{O}\left(\sqrt{H^2T\dim_E\left(\F,T^{-1}\right)\left(\log\frac{HT}{\delta}+\log N(\F,T^{-2},\Vert\cdot\Vert_\infty)\right)}\right)
	\end{align*}
\end{proof}

\begin{theorem}\label{thm:eluder}
	With probability at least $1-\delta$, it holds that for any $T\ge 1$,
	$$
	\operatorname{Reg}(T)=\mathcal{O}\left(\sqrt{H^6T\dim_E\left(\F,T^{-1}\right)\left(\log\frac{HT}{\delta}+\log N(\F,T^{-2},\Vert\cdot\Vert_\infty)\right)}\right)
	$$
\end{theorem}

\begin{proof}
	By \cref{eluder_event_high}, $\mathfrak{E}_\text{high}$ holds with probability at least $1-\delta$, so we have
	\begin{align*}
		\operatorname{Reg}(T)&\le 4H\sum_{t=1}^T\sum_{h=1}^H\Delta_{l_h^{(t)}-1} \tag{\cref{Q_near_optim}}\\
		&=8H\sum_{h=1}^H\sum_{t=1}^T 2^{-l_h^{(t)}}\\
		&=\mathcal{O}\left(\sqrt{H^6T\dim_E\left(\F,T^{-1}\right)\left(\log\frac{HT}{\delta}+\log N(\F,T^{-2},\Vert\cdot\Vert_\infty)\right)}\right) \tag{\cref{eluder_bound_l}}
	\end{align*}
\end{proof}

\end{document}